\documentclass[journal]{IEEEtran}

\usepackage{dsfont}
\usepackage{color}
\definecolor{pink}{rgb}{0.58,0,0.83}
\definecolor{orange}{rgb}{1,0.5,0}
\definecolor{lightgreen}{rgb}{0.2, 0.8, 0.2}
\definecolor{lightyellow}{rgb}{0.84, 0.65, 0.13}

\usepackage[colorlinks,linkcolor=blue,anchorcolor=blue,citecolor=red]{hyperref}
\usepackage{amsmath,amsfonts}
\usepackage{amsthm}
\usepackage{algorithmic}
\usepackage[linesnumbered, ruled]{algorithm2e}
\usepackage{array}
\usepackage{subfigure}
\usepackage{makecell}
\newtheorem{problem}{Problem}
\usepackage{textcomp}
\usepackage{stfloats}
\usepackage{url}
\usepackage{verbatim}
\usepackage{graphicx}
\usepackage{amssymb}
\usepackage{dsfont}
\usepackage{bbding}
\usepackage{booktabs}
\usepackage{multirow}
\usepackage{tablefootnote}
\usepackage{cite}
\usepackage{xcolor}
\usepackage{threeparttable}
\newtheorem{assumption}{Assumption}
\newtheorem{definition}{Definition} 
\newtheorem{remark}{Remark}
\newtheorem{lemma}{Lemma}
\newtheorem{thm}{Theorem}
\usepackage{arydshln}
\usepackage{multicol}
\definecolor{ForestGreen}{RGB}{34,139,34}
\def\BibTeX{{\rm B\kern-.05em{\sc i\kern-.025em b}\kern-.08em
		T\kern-.1667em\lower.7ex\hbox{E}\kern-.125emX}}
\usepackage{balance}
\begin{document}
	\title{\huge Vector Field-Guided Learning Predictive Control for Motion Planning of Mobile Robots with Uncertain Dynamics}
	
		%
	
	\author{
		Yang Lu$^{1}$, Weijia Yao$^{2}$, \emph{IEEE Member}, Yongqian Xiao$^{1}$, Xinglong Zhang$^{1}$, \emph{IEEE Member}\\ Xin Xu$^{1}$, \emph{IEEE Senior Member}, Yaonan Wang$^{2}$, Dingbang Xiao$^{1}$
		\thanks{$^{1}$Yang Lu, Yongqian Xiao, Xinglong Zhang, Xin Xu, and Dingbang Xiao are with the College of Intelligence Science and Technology, National University of Defense Technology, Hunan, China {\tt\footnotesize luyang18@mail.sdu.edu.cn, shawyongqian@gmail.com, zhangxinglong18@nudt.edu.cn, xuxin$\_$mail@263.net, dingbangxiao@nudt.edu.cn}.}
		\thanks{$^{2}$Weijia Yao and Yaonan Wang are with the School of Robotics, Hunan University, China {\tt\footnotesize wjyao@hnu.edu.cn,yaonan@hnu.edu.cn}. }
		\newline
		\textcolor{red}{This work has been accepted by IEEE/ASME Transactions on Mechatronics for publication. Copyright may be transferred without notice, after which this version may no longer be accessible.}
	}
	
	\maketitle
	\thispagestyle{empty}
	\pagestyle{empty}
	
	\begin{abstract}
		In obstacle-dense scenarios,  providing safe guidance for mobile robots is critical to improve the safe maneuvering capability.
		However, the guidance provided by standard guiding vector fields (GVFs) may limit the motion capability due to the improper curvature of the integral curve when traversing obstacles. On the other hand, robotic system dynamics are often time-varying, uncertain, and even unknown during the motion planning process. Therefore, many existing kinodynamic motion planning methods could not achieve satisfactory reliability in guaranteeing safety. To address these challenges, we propose a two-level Vector Field-guided Learning Predictive Control (VF-LPC) approach that improves safe maneuverability. The first level, the guiding level, generates safe desired trajectories using the designed kinodynamic GVF, enabling safe motion in obstacle-dense environments. The second level, the Integrated Motion Planning and Control (IMPC) level, first uses a deep Koopman operator to learn a nominal dynamics model \emph{offline} and then updates the model uncertainties \emph{online} using sparse Gaussian processes (GPs). The learned dynamics and a game-based safe barrier function are then incorporated into the LPC framework to generate near-optimal planning solutions. Extensive simulations and real-world experiments were conducted on quadrotor unmanned aerial vehicles and unmanned ground vehicles, demonstrating that VF-LPC enables robots to maneuver safely.
	\end{abstract}
	
	
	
	\definecolor{limegreen}{rgb}{0.2, 0.8, 0.2}
	\definecolor{forestgreen}{rgb}{0.13, 0.55, 0.13}
	\definecolor{greenhtml}{rgb}{0.0, 0.5, 0.0}
	
	\section{Introduction}
	\IEEEPARstart{P}{rior} information or desired paths are often required to guide robots' motion.
	As an effective and efficient method, guiding vector field (GVF) techniques have been recently studied to realize path-following or obstacle-avoidance tasks successfully for robots like fixed-wing airplanes\cite{yao2021singularity}, unicycle-type vehicles\cite{panagou2014motion},  unmanned aerial vehicles\cite{marchidan2020collision}, etc. The GVF typically addresses motion planning and control tasks in the following manner\cite{7942030}: (1) simple kinematic models of robots are considered (such as single or double integrator models); (2) GVF provides guidance signals to be tracked by the inner-loop dynamic controller). Therefore, the guidance level (i.e., the design of the GVF) and the control level can be separately designed. 
	However, the proposed GVFs such as\cite{yao2022guiding} may generate guidance signals with drastic variations when crossing static obstacles, posing a challenge to the robot's maneuvering capabilities. On the other hand, robot dynamics and safety constraints (e.g., avoidance of suddenly appearing obstacles) are critical in improving the robots' maneuvering capabilities. When considering the robotic dynamics constraint separately at planning and control levels, the issues of consistency and optimality have to be carefully addressed for real-world applications. Therefore, it is crucial to introduce an Integrated Motion Planning and Control (IMPC) approach to generate solutions transmitted to robot dynamics directly. In light of the above two aspects, designing adaptive IMPC approaches that leverage the improved GVF as guidance for real-world robots with uncertain/unknown dynamics is promising, particularly in obstacle-dense scenarios; An illustration diagram is presented in Fig.~\ref{fig_kinematics_dynamics_v3}.

	\begin{figure}[!htb]
		\centering\includegraphics[width=3.0in]{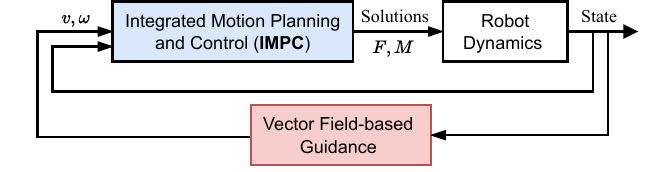}
		\caption{The proposed integrated motion planning and control architecture. In this figure, the guidance module generates desired linear and angular velocities $v$ and $\omega$; $F$ and $M$ are forces and moments used to control a robot directly. }
		\label{fig_kinematics_dynamics_v3}
	\end{figure}
	
	Optimization-based motion planning methods have recently been studied for realizing IMPC \cite{nakka2022trajectory,pek2020fail}. It is crucial to estimate model uncertainties in real-time to address the challenge of uncertain robotic dynamics impacting model predictive control (MPC) methods. This can be achieved through online estimation techniques, while solving nonlinear or non-convex optimization problems online may pose reliability and computational intensity issues. Reinforcement learning (RL) and adaptive dynamic programming (ADP) are promising in solving optimal planning and control problems\cite{deptula2019approximate},\cite{10271561},\cite{9756946}. Recent works on realizing RL-based IMPC have been studied\cite{10271561,9756946}. They generally design proper reward functions based on priori information or desired paths to learn optimal policies. With the desired paths, existing endeavors primarily realize safe tracking control by designing reward functions encompassing state errors, control inputs, and safety terms. However, in obstacle-dense scenarios, obtaining the desired safe paths is not easy. When the desired path is constituted by a straight line from the starting point to the endpoint, it may traverse many obstacles in obstacle-dense environments. Furthermore, the desired speed or angular velocity for these paths needs to be computed separately. The two aspects increase the complexity of designing RL-based IMPC algorithms, posing challenges to achieving near-optimal performance. Therefore, providing guidance for them is very important. Despite the above difficulties, current RL-based IMPC has shown effectiveness and efficiency for robots with nonlinear system dynamics\cite{deptula2019approximate},\cite{10271561},\cite{9756946}. Letting RL-based IMPC approaches work as the IMPC structure in Fig.~\ref{fig_kinematics_dynamics_v3} is still promising if the above-discussed challenges can be well addressed. Namely, a GVF provides kinematic guidance (e.g. linear and angular velocities) for RL-based IMPC approaches.

	There exist two main categories of RL-based IMPC studies; i.e., model-based and model-free ones. Obtaining a precise dynamic model is nontrivial due to internal factors such as system nonlinearity, and external factors such as uneven terrain, and slippery surfaces. The advantage of model-free RL methods lies in their independence from precise models. However, they still lack generalization ability to the unseen scenarios and data-efficiency of policy learning. Efficient model-based RL (MBRL) approaches have shown their effectiveness in real-world applications\cite{10271561,9756946}. Since real-world robot dynamics are often uncertain and even unknown, MBRL approaches struggle to achieve satisfactory reliability. This is due to: (1) The online adaptability of current data-driven modeling methods is insufficient; (2) The safety terms in the reward function may lead to policy divergence of RL-based motion planning algorithms. Motivated by the two challenges, we propose compensating for the data-driven model online and adopting a receding-horizon actor-critic framework (also called Learning Predictive Control\cite{xu2018learning}, LPC) to ensure the convergence of planning policies in the prediction horizons.

	To realize motion planning of mobile robots with uncertain/unknown dynamics in obstacle-dense scenarios, we propose a Vector Field-guided Learning Predictive Control (VF-LPC) approach. Specifically, the guiding level plans preliminary trajectories fast to avoid dense static obstacles. The optimal performance under the uncertain/unknown system dynamics and the safety constraints from (suddenly appearing) moving obstacles are optimized by solving the Hamilton–Jacobi–Bellman (HJB) equation online in prediction horizons. In particular, we introduce a sparsification technique in the model compensation and finite-horizon actor-critic learning processes to improve online efficiency and IMPC performance. The contributions of this paper are summarized as follows:

	\begin{itemize}
		\item[1)] The VF-LPC approach can achieve near-optimal motion planning for mobile robots with uncertain/unknown dynamics in obstacle-dense environments. The approach achieves higher computational efficiency and obtains more reliable solutions than advanced model predictive control (MPC) and RL methods in solving nonlinear optimization problems (see Sections~\ref{SecIII-B-3} and~\ref{Comparision}).
		\item[2)] Our proposed VF-LPC approach not only has theoretical guarantees but also has been demonstrated effective in practice since it has been validated by extensive simulations and experiments with quadrotor unmanned aerial vehicles (UAVs) and a Hongqi E-HS3 vehicle (see Sections~\ref{Section_TheoreticAnalysis} and~\ref{Comparision}). 
		\item[3)] The VF-LPC approach can update online the uncertain dynamics of a fully data-driven model trained by the deep Koopman operator. It reduces the differences between the real and learned system dynamics models when the environment is time-varying or the system dynamics are learned inaccurately. Therefore, it improves the planning performance and guarantees safety (see Section~\ref{3B1}).
		\item[4)] By adding virtual obstacles, the modified and improved discrete-time composite vector field adopted by our VF-LPC approach can satisfy robot kinodynamic constraints. In addition, the modified vector field does not suffer from the deadlock problem, which usually exists in traditional composite vector fields.  Moreover,  VF-LPC can deal with (suddenly appearing) moving obstacles by introducing a game-based barrier function (see Sections~\ref{SFSCVF}, \ref{SwitchingBarrierFunction}). 
	\end{itemize}
	
	The remainder of this paper is organized as follows. Section~\ref{RelatedWork} reviews the related works. Then Section \ref{PreliminariesAndProblemFormulation} provides the preliminaries and problem formulation. The VF-LPC is introduced in Section \ref{GuidingRHRL}. Then, Section \ref{Section_TheoreticAnalysis} presents the convergence results of VF-LPC. Section \ref{Comparision} elaborates on the simulation and experimental validation. Finally, the conclusion is drawn in Section \ref{Conclusion}.
	
	\textit{Notation}: The notation $\left\|{x}\right\|_{Q}^{2}$ represents $x^{\top}Qx$, where $Q$ is a positive (semi-)definite matrix, and  $\left\|{x}\right\|=\sqrt{x^{\top}x}$. The field of real numbers is denoted by $\mathbb{R}$. A diagonal matrix is denoted by $\text{diag}\{\nu_1,\cdots,\nu_n\}$, where $\nu_1,\cdots,\nu_n \in \mathbb{R}$ are entries on the diagonal. Throughout this paper, we use the notation $I$ to denote the identity matrix of suitable dimensions. The notations $\otimes$ and $\odot$ denote the Kronnecker and Hadamard products, respectively.
	\section{Related Work}\label{RelatedWork}
	We first discuss several modeling methods and then present a literature review on MPC and learning-based approaches for motion planning of robots with uncertain dynamics.
	
	\emph{Data-driven modeling}. Current advanced data-driven modeling methods include least-squares\cite{vicente2020linear}, recurrent neural networks (NNs)\cite{8574962}, multi-layer perception (MLP)\cite{spielberg2019neural}, neural networks\cite{da2020modelling}, etc.  As a linear operator, the Koopman operator-based modeling methods\cite{schmid2010dynamic},\cite{williams2015data},\cite{kevrekidis2016kernel} can establish linear time-invariant system dynamics. Impressed by such a property, dynamic mode decomposition (DMD)\cite{schmid2010dynamic}, extended DMD (EDMD)\cite{williams2015data}, and kernel-based DMD\cite{kevrekidis2016kernel} have received much attention in recent years. However, the modeling performance of the Koopman operator relies heavily on the observable function design. Consequently, approaches using NNs for automated observable function construction \cite{lusch2018deep,otto2019linearly} were proposed and have been validated to be effective through numerical simulations. 
	To improve the modeling accuracy, Xiao et al. \cite{xiao2023ddk} proposed a deep direct Koopman (DDK) method for identifying linear time-invariant vehicle dynamic models. Unlike these, we further consider improving the online adaptability of offline-trained dynamics models by learning the uncertain dynamics of the offline-trained Koopman model.
	
	\emph{Sampling-based and optimization-based motion planning algorithms for robots with uncertain dynamics.}
	Several studies have integrated chance constraints with sampling-based Rapidly-exploring Random Trees (CC-RRT) methods, presenting efficient path planning capability in densely cluttered obstacle environments \cite{luders2010chance,aoude2013probabilistically}. Gaussian processes (GPs) were employed for determining dynamically feasible paths and CC-RRT for establishing probabilistically feasible paths\cite{aoude2013probabilistically}. These approaches fail to guarantee
	optimality due to the lack of consideration of the robot’s
	motion dynamics. 
	
	Under uncertain dynamics, optimization-based motion planning methods typically involve objective functions incorporating Conditional Value-at-Risk (CVaR) measures\cite{hakobyan2022distributionally},\cite{dixit2023risk}. Nakka et. al. handle the motion planning problem of chance-constrained nonlinear stochastic systems by deriving a surrogate problem with convex constraints\cite{nakka2022trajectory}. Zhu et. al. designed a chance-constrained nonlinear MPC method to solve collision avoidance problems of multi-robots under various uncertainties like motion disturbance\cite{zhu2019chance}. To realize real-time optimization, they developed a tight bound for the approximation of collision probability. In\cite{lew2020chance}, Lew et. al. proposed a robust trajectory optimization method for nonlinear systems with model uncertainty and disturbances. Especially, it is capable of processing non-convex obstacle constraints.
	In \cite{lindemann2021robust}, Kalman filtering was utilized for state estimation, and risk-aware safety constraints arising from estimation errors were introduced into stochastic optimal motion planning problems.  Many of the above studies process various uncertainties like sensing, obstacle motion, robot dynamics, etc. According to \cite{singh2023robust}, two major ideas are considered in the area of feedback motion planning, i.e., \emph{probabilistic guarantees on safety} and \emph{bounded models of uncertainties}. In this paper, the unknown system dynamics are \emph{modeled} with the previous deep Koopman operator\cite{xiao2023ddk}. We consider the bounded uncertainty of data-driven system dynamics and compensate online for it. 
	
	\emph{Learning-based motion planning for robots with uncertain dynamics.}
	As discussed in\cite{snyder2023online}, RL-based approaches to addressing obstacle avoidance problems under uncertainties typically fall into two categories: one involves the endeavor to construct stochastic models of the uncertain dynamics inherent in robotic systems, leveraging the resultant probabilistic models for planning or policy learning; the other entails devising plans that account for worst-case scenarios. Regarding the first category, a Gaussian process model was used in the policy search framework of PILCO, thereby capturing the system dynamics and estimating the probability of safe constraints violation\cite{polymenakos2019safe}. During the policy learning process, the candidate policies are optimized toward the safer directions with low risks. In \cite{janson2017monte}, the Monte Carlo motion planning method was proposed to sample feasible trajectories under uncertainties, thereby fulfilling probabilistic collision avoidance constraints. Model-based motion planning under uncertain dynamics can be found in \cite{10271561},\cite{9756946}.
	Regarding the second category, Snyder et. al. proposed a trust-region-based online learning algorithm with provable regret bounds by minimizing worst-case regret\cite{snyder2023online}. To realize kinodynamic motion planning and control, we design an IMPC framework leveraging the guidance from the vector field and further consider the uncertainties of the data-driven model in obstacle-dense environments.
	\section{Preliminaries and Problem Formulation}\label{PreliminariesAndProblemFormulation}
	
	This section presents a detailed preliminary for the composite vector field, which will be developed later in this paper for generating preliminary kinodynamic trajectories. Then, we review a data-driven deep Koopman-based system modeling method. The previously developed sparse GP can efficiently identify model uncertainties online, which is also reviewed here to identify the model uncertainties of the nominal model. Note that for one thing, GP can be used to identify the full system dynamics individually, but the long-horizon modeling accuracy is difficult to guarantee. To enhance the accuracy, one has to use flawless samples and fine-tune the hyperparameters, which can be computationally demanding. For another, the Koopman model may not accurately characterize the exact model, so estimating the uncertainty of model learning is essential. Therefore, to improve the accuracy, we propose to combine the Koopman model learning and the sparse GP. In particular, we employ online sparse GP to compensate for the inaccuracy and uncertainty associated with the offline-trained Koopman model, which will be introduced later in our methodology. Finally, we present the problem formulation for optimal motion planning under the fully data-driven system model containing uncertain dynamics.
	\subsection{Composite Vector Field}
	Consider the following ordinary differential equation
	\begin{equation*}
		\dot{\xi}=\chi(\xi(t))
	\end{equation*}
	with the initial state $\xi(0)\in\mathbb{R}^2$, where $\chi(\cdot)$ is continuously differentiable concerning $\xi$, and it is designed to be a guiding vector field for path following \cite{yao2022guiding}.
	
	In Fig.~\ref{fig_repulsiveAndreactiveArea}, the elements of the composite vector field are illustrated in detail. A reference path $\mathcal{P}$ is provided initially and may be occluded by obstacles, and it is defined by
	\begin{equation*}
		\mathcal{P}=\left\{\xi \in \mathbb{R}^2: \phi(\xi)=0\right\},
	\end{equation*}
	where $\phi:\mathbb{R}^2\rightarrow\mathbb{R}$ is twice continuously differentiable. For example, a circle path $\mathcal{P}$ can be described by choosing $\phi(x,y)=x^2+y^2-R^2$, where $R$ is the circle radius. To avoid collisions, Yao et al. \cite{yao2022guiding} proposed a composite vector field for processing obstacle constraints. It involves a reactive boundary $\mathcal{R}_i^t$ and a repulsive boundary $\mathcal{Q}_i^t$, i.e.,
	\begin{equation*}\small
		\begin{aligned}
			\mathcal{R}_i^t=\left\{\xi \in \mathbb{R}^2: \varphi_i(\xi, t)=0\right\}, \mathcal{Q}_i^t=\left\{\xi \in \mathbb{R}^2: \varphi_i(\xi, t)=c_i\right\},
		\end{aligned}
	\end{equation*}
	where $\varphi_i: \mathbb{R}^2 \times \mathbb{R} \to \mathbb{R}$ is twice continuously differentiable, $i\in\mathcal{I}=\{1,2,\cdots,m\}$, $m$ is the total number of obstacles, and $c_i<0$. The repulsive boundary $\mathcal{Q}_i^t$ is the boundary that tightly encloses the $i$-th obstacle at time $t$ and a robot is forbidden to cross this boundary to avoid collision with the obstacle. The reactive boundary $\mathcal{R}_i^t$ is larger than and encloses the repulsive boundary, and its interior is a region where a robot can detect an obstacle and become reactive. We use prescripts ${\rm{ex}}$ and ${\rm{in}}$ to denote the exterior and interior regions of a boundary, respectively. For example, $^{\rm{ex}}{\mathcal{Q}}$ represents the exterior region of the repulsive boundary (see Fig.~\ref{fig_repulsiveAndreactiveArea}). An example of moving circular reactive and repulsive boundaries can be characterized by choosing $\varphi_i(x,y;t)=(x-t)^2+y^2-R^2$ and letting $|c_i|<R$. In this case, the reactive and repulsive boundaries are large and small concentric circles moving along the $x$-axis as $t$ increases, respectively.

	\begin{figure}[!htpb]
		\centering\includegraphics[width=3.0in]{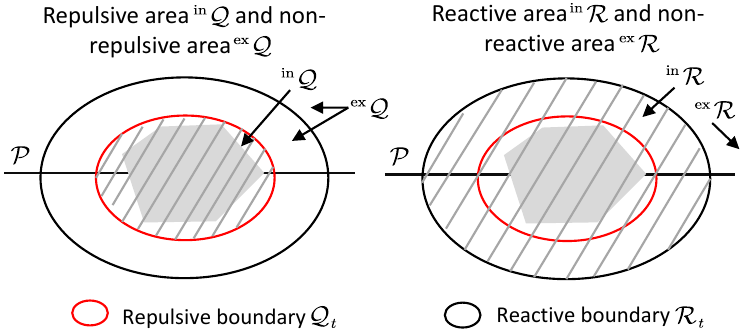}
		\caption{Component illustration of the composite vector field\cite{yao2022guiding}. The oblique lines construct a close area. The annulus area is the \emph{sandwiched region} $\mathcal{M}_s={}^{\mathrm{ex}} \mathcal{Q} \cap {}^{\mathrm{in}} \mathcal{R}$.}
		\label{fig_repulsiveAndreactiveArea}
	\end{figure}
	
	We denote the \emph{path-following vector field} by $\chi_{\mathcal{P}}$ and the \emph{repulsive vector field} by $\chi_{\mathcal{R}_i}$, and they are defined below:
	\begin{equation*}
		\begin{aligned}
			\chi_{\mathcal{P}}(\xi) &=\gamma_0 E \nabla \phi(\xi)-k_p \phi(\xi) \nabla \phi(\xi), \\
			\chi_{\mathcal{R}_i}(\xi) &=\gamma_i E \nabla \varphi_i(\xi)-k_{r_i} \varphi_i(\xi) \nabla \varphi_i(\xi), \;i \in \mathcal{I},
		\end{aligned}
	\end{equation*}
	where $E=\left[ \begin{smallmatrix}
		0&		-1\\
		1&		0\\
	\end{smallmatrix} \right] $ is the rotation matrix of $90^{\circ}$, $\gamma_i\in\{1,-1\}$, $i \in \mathcal{I}\cup\{0\}$, determines the moving direction (clockwise or counterclockwise), and $k_p, k_{r_i}$ are positive coefficients. The composite vector field is as follows\cite{yao2022guiding}:
	\begin{equation} \label{eq_yaocvf}
		\chi(\xi)=\left(\prod_{i \in \mathcal{I}} \sqcup_{\mathcal{Q}_i}(\xi)\right) \hat{\chi}_{\mathcal{P}}(\xi)+\sum_{i \in \mathcal{I}}\left(\sqcap_{\mathcal{R}_i}(\xi) \hat{\chi}_{\mathcal{R}_i}(\xi)\right),
	\end{equation}
	where $\hat{(\cdot)}$ is the normalization notation (i.e., for a nonzero vector $v\in\mathbb{R}^n$, $\hat{v}=v/\lVert v\rVert$), $\sqcup_{\mathcal{Q}}(\xi)=\frac{f_1(\xi)}{f_1(\xi)+f_2(\xi)}, \sqcap_{\mathcal{R}}(\xi)=\frac{f_2(\xi)}{f_1(\xi)+f_2(\xi)}$ are smooth \emph{bump functions},	where $f_1(\xi)= 0$ if $\varphi(\xi) \leq c$ and $f_1(\xi)=\exp \left(\frac{l_1}{c-\varphi(\xi)}\right)$ if $\varphi(\xi)>c$,  $f_2({\xi})=\exp \left(\frac{l_2}{\varphi(\xi)}\right)$ if $\varphi(\xi)<0$ and $f_2({\xi})=0$ if $\varphi(\xi)\geq0$,	and $l_1,l_2>0$ are coefficients for changing the decaying or increasing rate. Note that for simplicity, the subscripts $i$ of related symbols have been omitted above. These smooth bump functions blend parts of different vector fields and create a composite vector field for path following and collision avoidance; for more details, see \cite{yao2022guiding}. To understand the composite vector field intuitively, we illustrate the composite vector field~\eqref{eq_yaocvf} in Fig.~\ref{fig_IllustrationCVF}. For $\chi(\xi)$ in~\eqref{eq_yaocvf}, it is equal to $\hat{\chi}_{\mathcal{P}}(\xi)$, $\sqcup _{\mathcal{Q}}\left( \xi \right) \hat{\chi}_{\mathcal{P}}\left( \xi \right) +\sqcap _{\mathcal{R}}\left( \xi \right) \hat{\chi}_{\mathcal{R}}\left( \xi \right)$, and $\hat{\chi}_{\mathcal{R}}(\xi)$ within the three regions $^{\rm{ex}}{\mathcal{R}}$, $^{\rm{ex}}{\mathcal{Q}}\cap{}^{\mathrm{in}} \mathcal{R}$, and ${}^{\mathrm{in}} \mathcal{Q}$, respectively.
	
	\begin{figure}[!htb]
		\centering\includegraphics[width=2.0in]{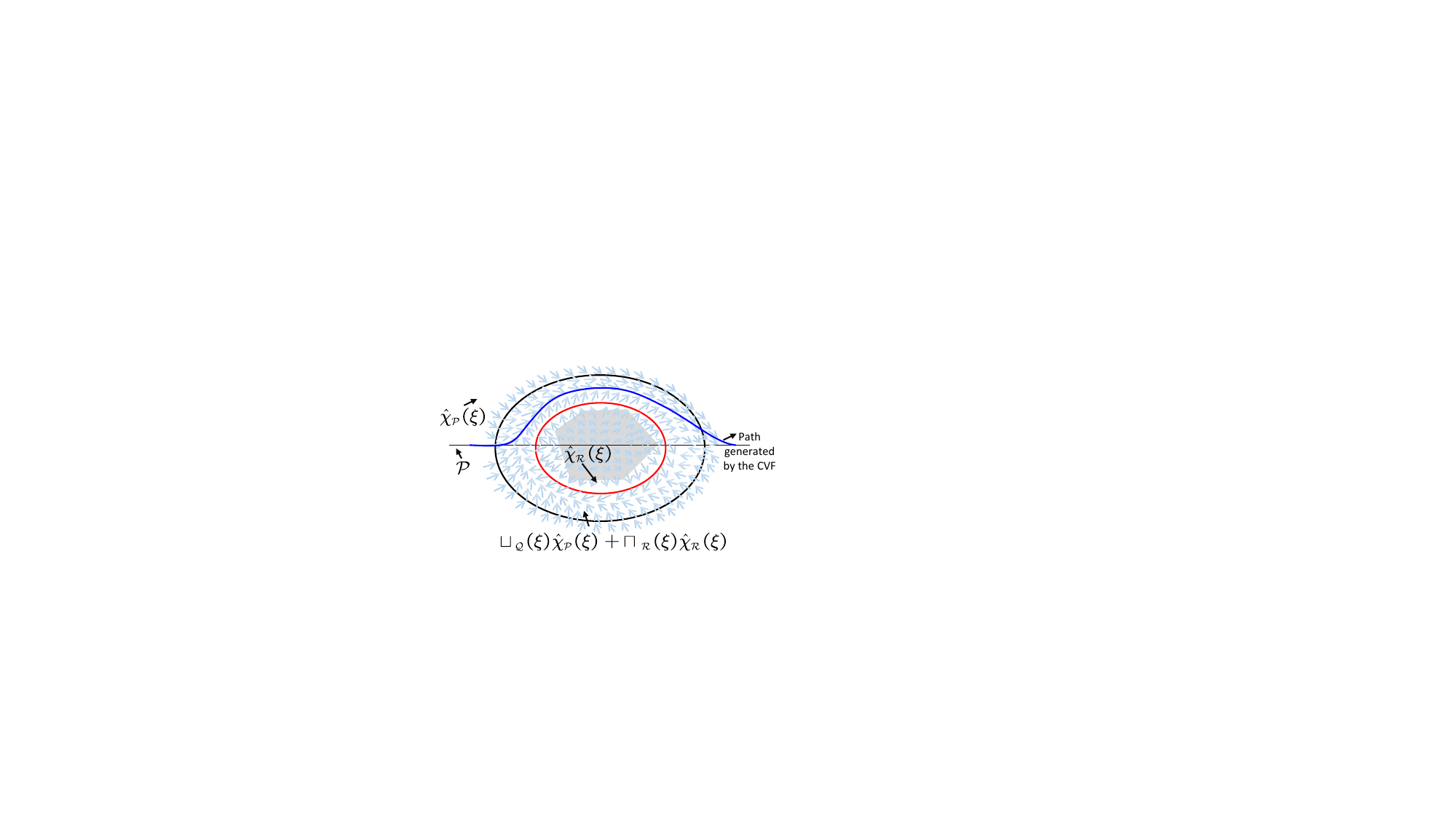}
		\caption{An illustration of the composite vector field~\eqref{eq_yaocvf}. We primarily illustrate the directions of the vector field in the vicinity of the path and directions partly within the interior region $^{\rm{in}}{\mathcal{Q}}$.}
		\label{fig_IllustrationCVF}
	\end{figure}
	
	The \emph{singular sets} of $\chi_{\mathcal{P}}$ and $\chi_{\mathcal{R}_i}$, denoted by $\mathcal{C}_{\mathcal{P}}$ and $\mathcal{C}_{\mathcal{R}_i}$, respectively, are defined below:
	\begin{equation*}
		\begin{aligned}
			\mathcal{C}_{\mathcal{P}} & =\left\{\xi \in \mathbb{R}^2: \chi_{\mathcal{P}}(\xi)=0\right\}=\left\{\xi \in \mathbb{R}^2: \nabla \phi(\xi)=0\right\}, \\
			\mathcal{C}_{\mathcal{R}_i} & =\left\{\xi \in \mathbb{R}^2: \chi_{\mathcal{R}_i}(\xi)=0\right\}=\left\{\xi \in \mathbb{R}^2: \nabla \varphi_i(\xi)=0\right\}.
		\end{aligned}
	\end{equation*}
	The elements of singular sets are called \emph{singular points}, where vector fields vanish. Due to the possible presence of singular points, special designs are required to solve the deadlock problem. When we employ the guiding vector field as high-level guiding signals, neglecting the kinodynamic constraints usually deteriorates the control performance. Therefore, in the subsequent sections, we will consider the kinodynamic constraints. 
	\subsection{Offline Deep Koopman Operator-based System Modeling}
	Consider the following continuous-time nonlinear system
	\begin{equation}\label{nonlienar_nonaffline_system}
		\dot{x}=f(x,u), 
	\end{equation}
	where $x\in\mathbb{R}^{n_x}$ denotes the system state, $f:\mathbb{R}^{n_x} \times \mathbb{R}^{n_u} \rightarrow\mathbb{R}^{n_x}$ is the system transition function, $u\in \Omega_u\subset\mathbb{R}^{n_u}$ denotes the control input, and $\Omega_u$ is the control constraint set. Note that the explicit dependence on time is dropped unless needed
	for clarity. We assume that $f(\cdot)$ is locally Lipschitz continuous.
	
	The discrete-time Koopman operator of~\eqref{nonlienar_nonaffline_system} can be described as follows:
	\begin{equation}
		\varUpsilon_{\infty}(x_{k+1}) = (\mathcal{K} \varUpsilon_{\infty})(x_k, u(x_k)),
	\end{equation}
	where $\mathcal{K}$ is an infinite-dimensional linear Koopman operator in a Hilbert space $\mathcal{H}$, and $\varUpsilon_{\infty}$ is the observable function. In \cite{xiao2023ddk}, $\mathcal{K}$ is approximated by $n_{\mathcal{K}}$-order system dynamics by using deep neural networks, i.e.,
	
	\begin{equation}\label{equ:latent_vehicle_dynamics}
		\varUpsilon(x_{k+1}) = A\varUpsilon(x_{k}) + Bu(x_k),
	\end{equation}
	where $A\in \mathbb{R}^{n_{\mathcal{K}}\times n_{\mathcal{K}}}$ and $B\in \mathbb{R}^{n_{\mathcal{K}}\times n_u}$ are latent system matrices, $\varUpsilon(x_k) = [x_k^{\top}, \rho_e^{\top}(x_k)]^{\top}$
	where $\rho_e: \mathbb{R}^{n_x} \to \mathbb{R}^{n_{\rho_e}}$ denotes the encoder module. 

	\subsection{Sparse GP Regression for Online Compensation}
	Next, we will review a sparse GP regression method called FITC \cite{snelson2005sparse}, which reduces computational complexity by selecting inducing samples and introduces a low-rank approximation of the covariance matrix, transforming the original GP model into an efficient one. It is briefly introduced in the following.
	\subsubsection{The formulation of full GP Regression}
	An independent training set is composed of state vectors, i.e., $\mathbf{z}=[z_1,z_2,\cdots,z_n]^{\top}\in\mathbb{R}^{n\times n_z}$ 
	and the corresponding output vectors $\mathbf{y}=[y_1,y_2,\cdots,y_n]^{\top}\in\mathbb{R}^{n\times n_{y}}$. In \cite{rasmussen2003gaussian}, the mean and variance functions of each output dimension $a\in\{1,\cdots,n_{y}\}$ at a test point $z=[x^\top,u^\top]^\top$ are computed by
	\begin{equation}\label{mean_covariance}
		\begin{aligned}
			&m^{a}_{d}=K^{a}_{z\mathbf{z}}(K^{a}_{\mathbf{zz}}+\sigma^2_{a}I)^{-1}[\mathbf{y}]_a,\\
			&\Sigma^{a}_{d}=K^{a}_{zz} -K^{a}_{z\mathbf{z}} \left( K^{a}_{\mathbf{zz}} +\sigma _{a}^{2}I \right) ^{-1}K^{a}_{\mathbf{z}z},\\
		\end{aligned}
	\end{equation}
	where $\sigma_{a}$ is the variance, $K_{\mathbf{zz}}^a=k^a(\mathbf{z},\mathbf{z})\in\mathbb{R}^{n_z\times n_z}$ is a Gram matrix containing variances of the training samples. Correspondingly, $K_{z\mathbf{z}}^a=(K_{\mathbf{z}z}^a)^{\top}=k^a(z,\mathbf{z})$ denotes the variance between a test sample and training samples, and $K_{zz}^a =k^a(z,z)$ represents the  covariance,
	$k^{a}(\cdot,\cdot)$ is the squared exponential kernel function and is defined as follows:
	\begin{equation}
		k^a\left(z_i, z_j\right)=\sigma_{f, a}^2 \exp (-1/2\left(z_i-z_j\right)^{\top} L_a^{-1}\left(z_i-z_j\right)),	
	\end{equation}
	where $\sigma_{f, a}^2$ is the signal variance and $L_a=\ell^2 I$. Here $\sigma_{f,a}$ and $\ell$ are hyperparameters of the covariance function.
	\subsubsection{Sparse GP Regression}
	Given an inducing dictionary set $\{\mathbf{z}_{\rm{ind}},\mathbf{y}_{\rm{ind}}\}$ with $n_{\rm{ind}}$ samples from $\{\mathbf{z},\mathbf{y}\}$, the prior hyper-parameters can be optimized by maximizing the marginal log-likelihood of the observed samples. In \cite{snelson2005sparse}, the mean and variance functions of a full GP are approximated by using inducing targets $\mathbf{y}_{\rm{ind}}$, inputs $\mathbf{z}_{\rm{ind}}$, i.e.,
	\begin{equation}\label{FITC}
		\begin{aligned}
			\tilde{m}_d^a(z)&=Q_{z \mathbf{z}}^a(Q_{\mathbf{z} \mathbf{z}}^a+\Lambda)^{-1}[\mathbf{y}]_{a}, \\
			\tilde{\Sigma}_d^a(z)&=K_{z z}^a-Q_{z \mathbf{z}}^a(Q_{\mathbf{z} \mathbf{z}}^a+\Lambda)^{-1} Q_{\mathbf{z} z}^a,
		\end{aligned}
	\end{equation}
	where $\Lambda=K_{\mathbf{z}\mathbf{z}}^a-Q_{\mathbf{z}\mathbf{z}}^a+\sigma_a^2I$ is diagonal and the notation $ Q_{\zeta \tilde{\zeta}}^a:=K_{\zeta \mathbf{z}_{\rm{ind}}}^a(K_{\mathbf{z}_{\rm{ind}} \mathbf{z}_{\rm{ind}}}^a)^{-1} K_{\mathbf{z}_{\rm{ind}} \tilde{\zeta}}^a$. Several matrices in~\eqref{FITC} do not depend on $z$ and can be precomputed, such that they only need to be updated when updating $\mathbf{z}_{\rm{ind}}$ or $\mathcal{D}$ itself.
	
	Finally, a multivariate GP is established by
	\begin{equation}
		d(z) \sim \mathcal{N}(\tilde{m}_d,\tilde{\Sigma}_{d}),
	\end{equation}
	where $\tilde{m}_d=[\tilde{m}_d^1,\cdots,\tilde{m}_d^{n_y}]^{\top}$, and $\tilde{\Sigma}_{d}=\text{diag}\{\tilde{\Sigma}_{d}^1,\cdots,\tilde{\Sigma}_{d}^{n_y}\}$.
	\subsection{Problem Formulation}

	\subsubsection{Composite Vector Field with Kinodynamic Constraints}
	The composite vector field acts as a local path planner and should satisfy the kinodynamic constraint, leading to the problem of Vector-Field-guided Trajectory
	Planning with Kinodynamic Constraint (VF-TPKC), which is formulated in Definition~\ref{definition1}. {This problem is decomposed into two components. In the presence of obstacles obstructing the desired path, the planning method should ensure the safety of paths, i.e., avoiding collision with obstacles. Then, the issue of satisfying dynamic constraints arises, involving improvements upon the path planning method established in the first step.}
	\begin{remark}
		The term ``kinodynamic constraint" refers to the requirement that a robot will not collide with obstacles at different speeds. To address this issue, we have transformed it into the fulfillment of the maximum lateral acceleration. $\hfill\blacktriangleleft$
	\end{remark}
	\begin{definition}\label{definition1}
		({VF-TPKC}) Design a continuously differentiable vector field $\chi:\mathbb{R}\times\mathbb{R}^2\rightarrow\mathbb{R}^2$ for $\xi(t)=\chi(t,\xi(t))$  such that:
		\begin{itemize}
			\item[1)] It achieves path-following and collision avoidance. In addition, the path-following error is bounded, and no deadlocks exist.
			\item[2)] Given the robot's velocities $v_x$, $v_y$, and the maximum centripetal acceleration $a_{\rm{max}}$, it holds that $({v_x^2+v_y^2})\kappa(t)\le a_{\rm{max}}$ for $t>0$ and $\kappa(t)$ is the curvature at time t. .
		\end{itemize}
	\end{definition}
	
	A guiding vector field $\chi:\mathbb{R}\times\mathbb{R}^2\rightarrow\mathbb{R}^2$ is designed to generate a continuously differentiable reference path, which is obtained by
	\begin{equation*}
		\begin{array}{l}
			\Xi=\int_0^\infty \chi(\xi(t))\rm{d}t.
		\end{array}
	\end{equation*}
	
	Subsequently, we employ a learning-based predictive control approach to track the desired trajectories and avoid dynamic obstacles at the same time.
	\subsubsection{Optimal Motion Planning to Avoid Moving Obstacles}\label{Problem2}
	Given the offline learned system~\eqref{equ:latent_vehicle_dynamics}, it is feasible to use it to design optimal IMPC. However, the interaction environments may be time-varying, causing the system dynamics to be uncertain. We can rewrite the exact system dynamics as a data-driven Koopman model adding an uncertain part by
	\begin{equation}\label{equ:latent_vehicle_dynamics_uncertain}
		\varUpsilon(x_{k+1}) = \underbrace{A\varUpsilon(x_{k}) + Bu(x_k)}_{f_{\text{nom}}(\varUpsilon(x_k),u(x_k))} +B_s\underbrace{\left( g(\varUpsilon(x_k),u(x_k))+w_k \right)}_{y_k},
	\end{equation}
	where the above model consists of a known nominal part $f_{\text{nom}}$ and an additive term $y_k$, which lies within the subspace spanned by $B_s$\cite{hewing2019cautious}. We assume that the process noise $w_k \sim \mathcal{N}\left(0, \Sigma^w\right)$ is independent and identically distributed (i.i.d.), with spatially uncorrelated properties, i.e., $\Sigma^w=\text{diag}\{\sigma_1^2, \ldots, \sigma_{n_y}^2\}$, where $n_y$ denotes the dimension of $y_k$.
	
	Assuming that the desired trajectory can be denoted by
	\begin{equation}\label{nominal_desired_path}
		\varUpsilon(x_{r,k+1}) = A\varUpsilon(x_{r,k}) + Bu(x_{r,k}),
	\end{equation}
	the subtraction of Eq.~\eqref{nominal_desired_path} from Eq.~\eqref{equ:latent_vehicle_dynamics_uncertain} yields the following error model, i.e.,
	\begin{equation}\label{discrete-time-nonlinear-system-dynmaics}
		\tilde{x}_{k+1}=A\tilde{x}_k + B\tilde{u}(x_k) +B_s\underbrace{\left( g(\varUpsilon(x_k),u(x_k))+w_k \right)}_{y_k},
	\end{equation}	
	where $\tilde{x}_k=\varUpsilon(x_k)-\varUpsilon(x_{r,k})$ is the error state, $x_{r,k}$ is the reference state, and $\tilde{u}(x_k)=u(x_k)-u(x_{r,k})$ is the control input. 
	
	We formally define the Optimal Motion Planning (OMP) problem, which consists of two subproblems. The first subproblem is the tracking control problem. To formulate this subproblem, we first define the value function as the cumulative discounted sum of the infinite-horizon costs:
	\begin{equation}\label{value_function}
		V_{\infty}(\tilde{x}_k)=\sum_{\tau=k}^{\infty}\gamma^{\tau-k}L(\tilde{x}_\tau,\tilde{u}(\tilde{x}_\tau)),
	\end{equation}
	where $0<\gamma\le1$, $L(\tilde{x}_\tau, \tilde{u}_\tau)=\tilde{x}_\tau^{\top}Q\tilde{x}_\tau+\tilde{u}^\top(\tilde{x}_\tau) R\tilde{u}(\tilde{x}_\tau)$ is the cost function, $Q\succeq 0\in\mathbb{R}^{n_{\mathcal{K}}\times n_{\mathcal{K}}}$ is positive semi-definite, and $R\succ 0\in\mathbb{R}^{n_u\times n_u}$ is positive definite.
	The second subproblem is how the robot can avoid moving obstacles.
	Combining the two subproblems, we formulate the optimal OMP as below:
	\begin{problem}
		({{OMP}}) Design an optimal IMPC for the robot with uncertain system dynamics such that it
		\begin{itemize}
			\item[{C.1}:] Starts at $x_0\in\mathbb{R}^{n_x}$ and tracks the reference path $\Xi$ by minimizing the value function $V_{\infty}(\tilde{x}_k)$. 
			\item[{C.2}:] Avoids collisions with all obstacles $\mathcal{B}_1,\dots, \mathcal{B}_q\subseteq\mathcal{W}\subseteq\mathbb{R}^2$, where $\mathcal{W}$ denotes the workspace.
		\end{itemize}
	\end{problem}
	
	\section{Vector Field Guided Receding Horizon Reinforcement Learning for Mobile Robots with Uncertain System Dynamics}\label{GuidingRHRL}
	It is essential to generate local collision-free trajectories for guiding robots' motion in obstacle-dense scenarios, which could improve safety and simplify the design of RL-based IMPC. Motivated by this aspect, we design a guiding vector field that considers dynamic constraints and excludes the deadlock problem (i.e., singular points). This part is illustrated by the \emph{safety guiding module} in Fig.~\ref{fig_DRHACL_VehicleControl}. Considering safety when robots track the desired trajectory, we must deal with the movements of (suddenly appearing) moving obstacles and the uncertainties of the nominal deep Koopman model. To this end, we develop an online receding-horizon reinforcement learning (RHRL) approach that employs a game-based exponential barrier function and a fast model compensation scheme. This part is illustrated by the \emph{learning predictive control module} of Fig.~\ref{fig_DRHACL_VehicleControl}. The details of each module and its sub-modules will be illustrated in the following subsections.

	\begin{figure*}[!htb]
		\centering\includegraphics[width=7.0in]{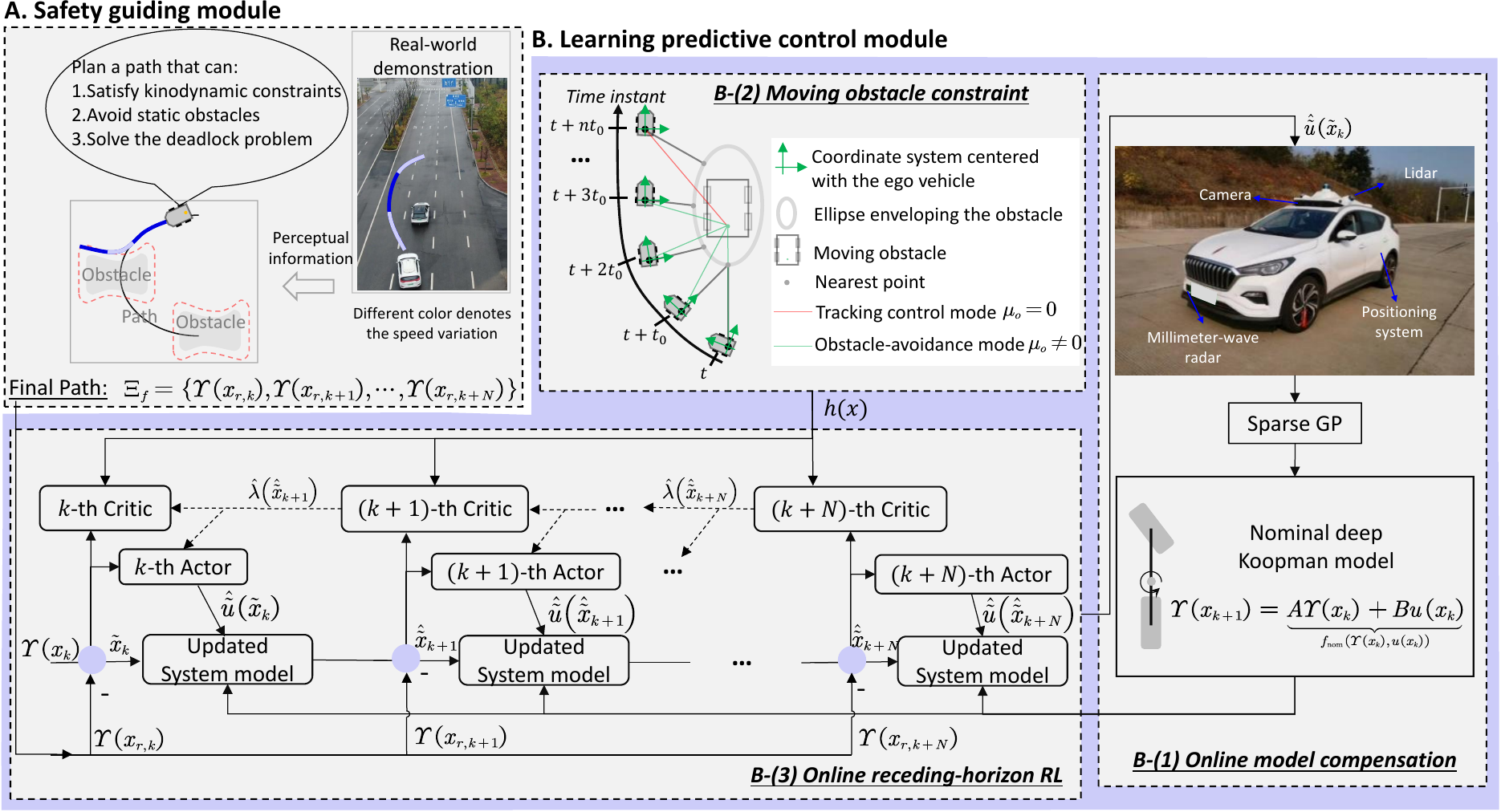}
		\caption{The overall framework of the VF-LPC algorithm.}
		\label{fig_DRHACL_VehicleControl}
	\end{figure*}
	\subsection{Discrete-time Kinodynamic Composite Vector Field}\label{SFSCVF}
	In this subsection, we present a discrete-time kinodynamic composite vector field to generate locally feasible trajectories, corresponding to module A of Fig.~\ref{fig_DRHACL_VehicleControl}.
	\subsubsection{Composite Vector Field with Kinodynamic Constraints}
	The first objective in Definition 1 can be achieved by the composite vector field~\eqref{eq_yaocvf}. 
	To accomplish the second objective in Definition \ref{definition1}, we first design the following kinodynamic composite vector field based on~\eqref{eq_yaocvf}:
	\begin{equation}\label{guiding_vector_dynamic_constraint} 
		\resizebox{0.85\hsize}{!}{$
			\begin{aligned}
				\chi_c(\xi)=&\left(\prod_{i \in \mathcal{I}\backslash\mathcal{I}'} \sqcup_{\mathcal{Q}_i}(\xi)\prod_{i \in \mathcal{I}'}s_i(\xi)\right) \hat{\chi}_{\mathcal{P}}(\xi)\\
				&+\sum_{i \in \mathcal{I}\backslash\mathcal{I}'}\left(\sqcap_{\mathcal{R}_i}(\xi) \hat{\chi}_{\mathcal{R}_i}(\xi)\right)+\sum_{i \in\mathcal{I}'}\left(1-{s_i}(\xi)\right) \hat{\chi}_{\mathcal{R}_i}(\xi),
			\end{aligned}$}
	\end{equation}
	where $\mathcal{I}'$ is a set containing the index numbers of manually added \emph{virtual obstacles}, and $s_i: \mathbb{R}^2 \to \mathbb{R}$ is a function to be designed later.		
	The path generated by the original composite vector field in \eqref{eq_yaocvf} would often require the robot to make large turns within a limited distance for collision avoidance. The role of \emph{virtual obstacles} here is to proactively modify the vector field such that the curvature of the robot trajectory is less than the maximum allowable value as the robot enters the \emph{sandwiched region} 
	\begin{equation}
		\mathcal{M}_s={}^{\mathrm{ex}} \mathcal{Q}_{\mathrm{actl}} \cap {}^{\mathrm{in}} \mathcal{R}_{\mathrm{actl}},
	\end{equation}
	(i.e., the area sandwiched between the repulsive and the reactive boundaries; see the white annulus region in Fig.~\ref{fig_safe_composite_field}), thereby satisfying the dynamic constraints. The definitions of the reactive and repulsive boundaries are expressed through the function $\varphi_i(\xi),i\in\mathcal{I}'$. For example, the reactive boundary $\mathcal{R}_{\rm{vrtl}}$ corresponding to a virtual obstacle is described by $\{\xi \in \mathbb{R}^2 : \varphi_i(\xi)=0\},i\in\mathcal{I}'$ in Fig.~\ref{fig_safe_composite_field}, and its repulsive boundary $\mathcal{Q}_{\rm{vrtl}}$ is described by $\{\xi \in \mathbb{R}^2 : \varphi_i(\xi)=c_i\},i\in\mathcal{I}'$. 
	When the robot's position $\xi=(X,Y) \in \mathbb{R}^2$ enters the virtual reactive region ${}^{\mathrm{in}}\mathcal{R}_{\rm{vrtl}}$, it will be attracted towards the virtual repulsive boundary $\mathcal{Q}_{\rm{vrtl}}$. This provides a direction change before $\xi$ enters the actual reactive region ${}^{\mathrm{in}}\mathcal{R}_{\rm{actl}}$ corresponding to the actual obstacle, and the virtual obstacle will not affect its motion after $\xi$ enters the actual reactive region ${}^{\mathrm{in}}\mathcal{R}_{\rm{actl}}$. Based on the above analyses, within/outside the \emph{buffer region} (i.e., the shaded area in Fig.~\ref{fig_safe_composite_field})
	\begin{equation}
		\mathcal{M}_b = {}^{\mathrm{ex}} \mathcal{R}_{\text{actl}}\cap {}^{\mathrm{in}} \mathcal{R}_{\text{vrtl}},
	\end{equation}
	the function ${s_i}(\xi):\mathbb{R}^2 \to \mathbb{R}, i\in\mathcal{I}'$ is designed to be 
	\begin{equation}\label{S_xi}
		{s_i}(\xi)= \begin{cases}  \exp \left(\frac{k_{c,i}}{c_i-\varphi_i(\xi)}\right) & \xi\in \mathcal{M}_b,\\ 1 & \text{otherwise},\end{cases}
\end{equation}
where the adjustable coefficient $k_{c,i}>0$ is used to change the convergence rate to $\mathcal{Q}_{\rm{vrtl}}$, and $\exp(\cdot)$ denotes the exponential function. Thus, in this design, $0<{s_i}(\xi)<1$ if $\xi\in\mathcal{M}_b$; ${s_i}(\xi)=1$, otherwise.

\begin{remark}\label{assumption:virtualobstacle}
The placed virtual obstacle is assumed to satisfy $({}^{\mathrm{in}}\mathcal{Q}_{\rm{vrtl}}\cap\mathcal{P})\subset({}^{\mathrm{in}}\mathcal{R}_{\rm{actl}}\cap\mathcal{P})$. The assumption is used for letting $\xi$ exit the {sandwiched region} $\mathcal{M}_s$ from $\mathcal{M}_b$, but not converging to $\mathcal{Q}_{\rm{vrtl}}$ when robots are in $\mathcal{M}_b$. $\hfill\blacktriangleleft$
\end{remark}
\begin{remark}
Virtual obstacles cease to exert their influence if $\xi$ enter $\mathcal{M}_b$ again from  $\mathcal{M}_s$. This setting stops robots from returning to $\mathcal{M}_s$, but enables them to move towards the desired path to complete an obstacle avoidance process.$\hfill\blacktriangleleft$
\end{remark}
Now, the guidance path generated by the vector field $\chi_c(\xi)$ can avoid rapidly increasing curvature. To satisfy the maximum centripetal acceleration
$a_{\rm{max}}$, \emph{speed planning} is further performed for the path $\Xi_f=\int_0^T \chi(\xi(t))\rm{d}t$ under a given desired speed $v_d$, where $T>0$ determines the time duration. The maximum allowable speed is $(a_{\rm{max}}/\kappa(t))^{1/2}$. We can perform \emph{speed planning} by the following strategy: If $v_d> (a_{\rm{max}}/\kappa(t))^{1/2}$, then set the speed at $\xi(t)$ to $(a_{\rm{max}}/\kappa(t))^{1/2}$; otherwise, set the speed at $\xi(t)$ to $v_d$.

\begin{figure}[h]
\centering\includegraphics[width=2.5in]{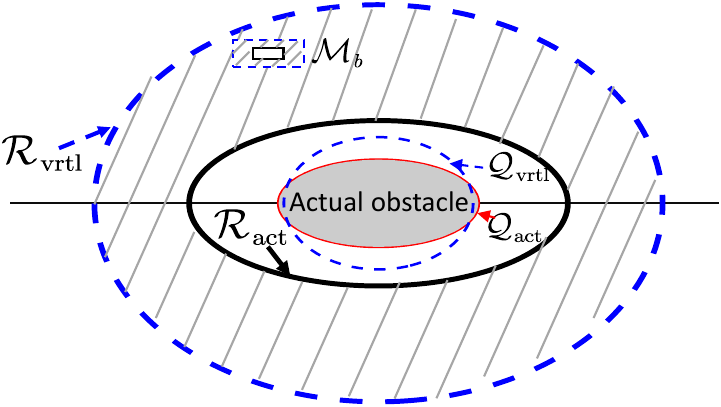}
\caption{An illustrative example of the kinodynamic composite vector field. The gray-shaded ellipse is the actual obstacle, and its corresponding repulsive and reactive boundaries are the red and black solid elliptic curves, respectively. The small and large blue circles with dashed lines represent the repulsive and reactive boundaries for a manually added virtual obstacle. The black curve is the desired path.}
\label{fig_safe_composite_field}
\end{figure}
\subsubsection{Analysis of the Composite Vector Field}
The composite vector field \eqref{guiding_vector_dynamic_constraint} gives an idea of how to accommodate the kinodynamic constraint. 
By redesigning the characterizing functions $\phi$, $\varphi_i$, the bump functions, and the coefficients $c_i$, the first condition of Definition~\ref{definition1} can be met in practice\cite{yao2022guiding}. The second condition can be satisfied by selecting proper positive coefficient $k_{c,i}, i\in\mathcal{I}'$. These two conditions are satisfied in simulations and experiments in this paper.
\subsubsection{Discrete-time Kinodynamic Guiding Vector Field} \label{subsec_dtgvf}
With $\chi_c(\xi)$ defined in~\eqref{guiding_vector_dynamic_constraint}, the desired path is computed by its discrete-time form, i.e.,
\begin{equation}\label{xi_r}
\xi_{k+1}= \xi_{k}+\beta\chi_c(\xi_{k, grid}), \quad \xi_0=(x_0,y_0),
\end{equation}
where $\xi_0\in\mathbb{R}^2$ is the initial robot position and $\beta$ is the step length, and $\xi_{k,grid} = \arg\min_{\xi_{g} \in \mathcal{G} } || \xi_k - \xi_g ||$, where the set $\mathcal{G} \subseteq \mathbb{R}^2$ is the grid map consisting of a finite number of chosen points in a selected region. In terms of computational efficiency, it is advisable to precompute $\chi_c(\xi)$ on a mesh grid map. This makes it possible to quickly identify the nearest vector $\chi_c(\xi_{k, grid})$ on the mesh grip map based on the current position $\xi_k$. Finally, we can obtain the planned trajectory. Note that a vector with nearly zero norms is not selected as the current vector but rather inherits the vector from the previous state to avoid suffering from the singularity issue.  Namely, if $|| \chi_c(\xi_{k,grid}) || < \epsilon$, then  $\chi_c(\xi_{k,grid}) \leftarrow  \chi_c(\xi_{k-1,grid})$, where $\epsilon$ is a small positive number. 
%
{This allows us to plan guiding trajectories directly without causing the deadlock problem from singular points.}

\subsection{Online Compensation to Update the Offline Trained Deep Koopman Model}\label{3B1} 
In this subsection, we present an online compensation method to update the offline trained deep Koopman model since \textbf{\emph{model uncertainties}} exist. This subsection corresponds to module B-(1) of Fig.~\ref{fig_DRHACL_VehicleControl}.

To construct the ``input-output" form of a GP, Eq.~\eqref{equ:latent_vehicle_dynamics_uncertain} is rewritten as
\begin{equation}\label{error_model}
{d}(z_k) = B_d^{\dagger}(\varUpsilon({x}_{k+1}) -f_{\text{nom}}\left(\varUpsilon({x}_k),{u}(x_k) \right)),
\end{equation}
where $B_d^{\dagger}$ is the Moore–Penrose pseudoinverse of $B_d$. 
\begin{remark}
Due to the necessity for efficient computation in the online compensation of vehicle dynamic models, we further apply the approximate linear dependence (ALD) strategy\cite{10271561} to quantize the online updated dictionary $\mathcal{D}_{\rm{GP}}$ for obtaining the training dataset $\mathcal{D}_{\rm{SGP}}$. Note that $\mathcal{D}_{\rm{GP}}$ is obtained in the same way as \cite{kabzan2019learning}. Subsequently, we employ the Sparse GP-based algorithm FITC \cite{snelson2005sparse} for online estimation to enhance computational efficiency. This approach aligns with our previous work, and additional details can be found in\cite{10271561}.$\hfill\blacktriangleleft$
\end{remark}

With the optimized parameters, GP models compensate for model uncertainties. Therefore, we define the learned model of~\eqref{equ:latent_vehicle_dynamics_uncertain} as
\begin{equation}\label{Dynamics_for_planning}
\varUpsilon({x}_{k+1})=f_{\text{nom}}(\varUpsilon({x}_k), {u}(x_k))+{d}(z_k)+\epsilon,
\end{equation}
where $\epsilon\in\mathbb{R}^{n_{\mathcal{K}}}$ is the estimation error. {The reference trajectory can be expressed as $\varUpsilon(x_{r,k+1})=f_{\text{nom}}(\varUpsilon({x}_{r,k}), {u}(x_{r,k}))$}.

We can obtain the Jacobian matrix of~\eqref{Dynamics_for_planning} at a reference state $(\varUpsilon({x}_{r,k}), {u}(x_{r,k}))$, i.e.,
\begin{equation}\label{JocabianMatrices}
A_{d,k}=A+\frac{\partial {d}(z_k)}{\partial \varUpsilon({{x}_k})}, B_{d,k}=B+\frac{\partial {d}(z_k)}{\partial {{u}_k}},
\end{equation}
where $A_{d,k}\in\mathbb{R}^{n_{\mathcal{K}}\times n_{\mathcal{K}}}$, $B_{d,k}\in\mathbb{R}^{n_{\mathcal{K}}\times n_u}$. 
Note that $A_{d,k}$ and $B_{d,k}$ are then used for the training process of model-based RL in Section~\ref{SecIII-B-3}. To obtain $\frac{\partial {d}(z_k)}{\partial\varUpsilon({{x}_k})}$ and $\frac{\partial {d}(z_k)}{\partial {{u}_k}}$, one should compute $\frac{\partial {d}(z_k)}{\partial {{z}_k}}$ first, which is equivalent to calculating $\partial\left(Q_{z_k\mathbf{z}_{\rm{sp}}}(Q_{\mathbf{z}_{\rm{sp}}\mathbf{z}_{\rm{sp}}}+\Lambda)^{-1}\mathbf{y}_{\rm{sp}}\right)/\partial z_k$, i.e., 
\begin{equation}
\begin{aligned}
	\dfrac{\partial {d}(z_k)}{\partial {{z}_k}}=\dfrac{1}{\ell^2}(\mathbf{z}_{\rm{sp}}-[{z}_k\otimes\mathds{1}])\left(Q_{z_k \mathbf{z}_{\rm{sp}}} \odot(Q_{\mathbf{z}_{\rm{sp}} \mathbf{z}_{\rm{sp}}}+\Lambda)^{-1}\mathbf{y}_{\rm{sp}}\right),
\end{aligned}
\end{equation}
where $\mathbf{z}_{\rm{sp}}\in\mathbb{R}^{n_z\times n_{\rm{sp}}}$ and $\mathbf{y}_{\rm{sp}}\in\mathbb{R}^{n_{\rm{sp}}\times n_{y}}$ denotes the input and output of the training samples $\mathcal{D}_{\rm{SGP}}$, respectively, $\mathds{1}$ denotes an $n_{\rm{sp}}$-dimensional row vector  consisting entirely of $1$. Ultimately, this yields matrices $\frac{\partial {d}(z_k)}{\partial \varUpsilon{({x}_k})}$ and $\frac{\partial {d}(z_k)}{\partial {{u}_k}}$.
Then an error model derived from~\eqref{discrete-time-nonlinear-system-dynmaics} is given below by assuming that the estimation error is sufficiently small, i.e.,
\begin{equation}\label{identification_form_dynmaic}
\hat{\tilde{x}}_{k+1}=f_{\text{nom}}(\tilde{x}_k,\tilde{u}_k(\tilde{x}_k)) + d({z_k}).
\end{equation}
\subsection{Exponential Barrier Function Incorporated Cost Function to Avoid Moving Obstacles}\label{SwitchingBarrierFunction}	
In this subsection, we design a cost function incorporated with an exponential barrier function to deal with the \textbf{\emph{moving obstacle constraint}}, corresponding to module B-(2) of Fig.~\ref{fig_DRHACL_VehicleControl}.

Note that {C.2} in the {{OMP}} Problem can cause safety issues due to the obstacles' movements. Consistent with \cite{10271561}, an implementable barrier function used for safety is designed as 
\begin{equation}\label{barrierfunction}
h\left( x \right) =\mu_o\exp \left( -\lVert \pi(x,x_p) \rVert \right),
\end{equation}
where $\mu_o\ge0$ is the penalty coefficient,  $x_p\in\mathbb{R}^2$ is the nearest reactive boundary coordinate from the nearest obstacle, and $\pi:\mathbb{R}^{n_x}\times\mathbb{R}^2\rightarrow\mathbb{R}$ maps $x$ and $x_p$ to the distance error. With $h(x)$, the robot keeps a safe distance from obstacles.

Due to the adoption of rule-based switching between tracking control and obstacle-avoiding modes in \cite{10271561}, without considering obstacle velocities, it does not guarantee the safety of mobile robots under extreme conditions such as the hazardous driving behavior of opposing vehicles. Our idea stems from the work\cite{exarchos2015suicidal} in completing suicidal game tasks, where the ``evader" aspires to avoid being caught by the ``pursuer" by computing safe actions. Similarly, when moving obstacles obstruct the road ahead, our strategy is to keep a safe distance from obstacles and track the desired trajectory if the ego robot cannot be caught. Integrating \cite{10271561} and \cite{exarchos2015suicidal}, the exponential barrier function comes into effect if the nearest surrounding obstacle would collide with the robot, where the judgment is generated by the method in \cite{exarchos2015suicidal}, as elaborated below.

\begin{figure}[!htb]
\centering\includegraphics[width=2.0in]{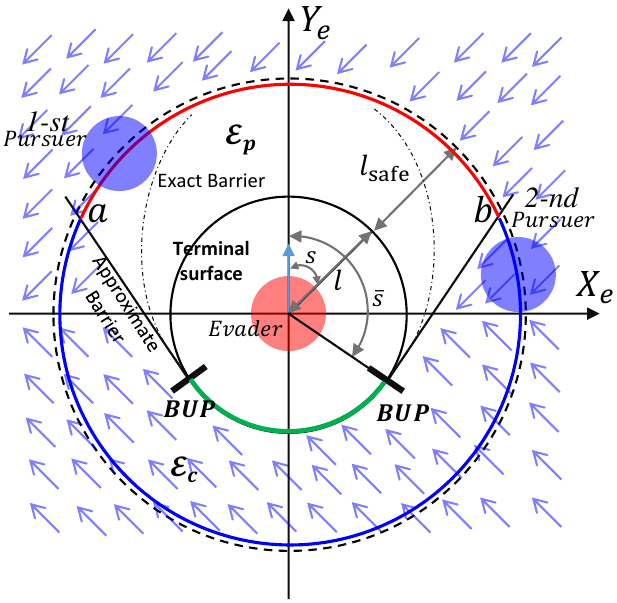}
\caption{Pursuer-Evader based game model. The barrier function comes into effect to ensure safety when the pursuer is situated within the $\mathcal{E}_p$ region.}
\label{fig_game}
\end{figure}

As shown in Fig.~\ref{fig_game}, a Cartesian coordinate system is established with the direction of the evader's velocity as the $Y$-axis, to facilitate the design of the safety-guaranteed switching mechanism. In this figure, $l_{\rm{safe}}$ denotes the preset maximum safety margin, and depends on the maximum tracking error of a controller, and $l$ is the safety distance defined by the physical constraints of the robot and obstacles. 
%
%
%
The terminal circle is divided into the \emph{usable part} (the black line) and the \emph{nonusable part} (the green line), separated by the boundary of the usable part (BUP). In \cite{exarchos2015suicidal}, the solution of BUP is determined by $\bar{s}=\arccos (-{v_p}/{v_e}),\ \bar{s} \in({\pi}/{2}, \pi]$,	where $v_p$ and $v_e$ are the speeds of the pursuer and the evader, respectively. The \emph{barrier} implies that the optimal play by both agents starting from any point will generate a path that does not penetrate this surface. The BUP connects the barrier and meets the terminal surface at the BUP tangentially. Therefore, as shown in Fig.~\ref{fig_game}, we consider the plane formed by points $a$, $b$, and the points at BUP as a conservative barrier to ensure safety.  
Fig.~\ref{fig_game} shows that the coordinate plane is divided into two regions: $\mathcal{E}_p= \mathcal{E}_{p,1}\cap\mathcal{E}_{p,2}$ and $\mathcal{E}_c= \mathcal{E} \setminus \mathcal{E}_p$, where $\mathcal{E}_{p,1}$ and $\mathcal{E}_{p,2}$ are defined below:
\begin{equation*}
\begin{aligned}
	&\mathcal{E}_{p,1}=\left\{ \zeta=[X_e,Y_e]^\top \in \mathbb{R}^2:\lVert \zeta\lVert \le l+l_{\text{safe}} \right\}\\
	&\mathcal{E}_{p,2}=\mathcal{H}_1 \cup \mathcal{H}_2 \cup \mathcal{H}_3, 
\end{aligned}
\end{equation*}
with
\begin{equation*}
\begin{aligned}
	&\mathcal{H}_1 = \{[X_e,Y_e]^\top : X_e\le -l\sin \bar{s}, \; Y_e\ge \tan (\bar{s})X_e+l/\cos \bar{s}\} \\
	&\mathcal{H}_2 = \{[X_e,Y_e]^\top : X_e\ge l\sin \bar{s}, \; Y_e\ge -\tan (\bar{s})X_e+l/\cos \bar{s}\} \\
	&\mathcal{H}_3 = \{[X_e,Y_e]^\top : X_e \in (-l\sin\bar{s},l\sin \bar{s}), \; Y_e\ge -\sqrt{l^2-X_e^2}\}.
\end{aligned}
\end{equation*}
%

The switching mechanism is elucidated more clearly below: In the area $\mathcal{E}_p$, $\mu_o\neq 0$ in \eqref{barrierfunction} is set to keep a safe distance from the obstacle, while in the region $\mathcal{E}_c$, we let $\mu_o=0$ to track the planned preliminary path.
To satisfy the conditions {C.1} and {C.2} in the {{OMP}} Problem, the step cost function at the $k$-th stage is re-designed as
\begin{equation}\label{cost_barrier}
\begin{aligned}
	L(\tilde{x}_k, \tilde{u}_k)=\tilde{x}_k^{\top}Q\tilde{x}_k+\tilde{u}_k^\top R\tilde{u}_k+h(x_k).
\end{aligned}
\end{equation}
{Equipped with the autonomous switching mechanism, the cost function is designed to learn motion planning policies.}


\subsection{Online Receding-horizon Reinforcement Learning with a Game-based Exponential Barrier Function Design}\label{SecIII-B-3}
In this subsection, we design a receding-horizon reinforcement learning with a game-based exponential barrier function, corresponding to module B-(3) of Fig.~\ref{fig_DRHACL_VehicleControl}.
In the previous subsection, the game-based exponential barrier function has been incorporated into the cost function to establish a constraint-free optimization problem, which is to be solved online in this subsection by the Receding-Horizon Reinforcement Learning (RHRL) approach. For this reason, we call it Exponential barrier function-based RHRL (Eb-RHRL).

To minimize the infinite-horizon value function $V_{\infty}(\tilde{x}_k)$, we first present the definitions of the optimal solution (control) and value functions commonly used in infinite-horizon ADP and RL.
As a class of ADP methods, the dual heuristic programming (DHP) approach aims to minimize the derivative of the value function with respect to the state. Consistent with the DHP-based framework, we re-define the value function as
\begin{equation}\label{optimal_lambda1}
\lambda_{\infty}(\tilde{x}_{k}) = \partial V_{\infty}(\tilde{x}_k)/\partial \tilde{x}_{k}.
\end{equation}

\begin{definition}
The optimal value function corresponding to the \textbf{infinite-horizon} optimization objective is defined by taking~\eqref{value_function} into~\eqref{optimal_lambda1}, i.e,
\begin{equation}\label{optimal_inftyVF}
	\lambda_{\infty}^*\left(\tilde{x}_k\right) 
	=2Q\tilde{x}_{k}+ \gamma A_{d,k}^\top \lambda^*(\hat{\tilde{x}}_{k+1})+\dfrac{\partial h({x}_k)}{\partial \tilde{x}_k}.
\end{equation}

Then, the optimal solution is defined by the following equation:
\begin{equation}\label{optimal_control_problem_inftyVF}
	\tilde{u}^*(\tilde{x}_k)=\arg\min_{\tilde{u}(\tilde{x}_k)}\sum_{\tau=k}^{\infty}\gamma^{\tau-k}L(\hat{\tilde{x}}_\tau,\tilde{u}(\hat{\tilde{x}}_\tau)),
\end{equation}
where $\gamma\in\left(0,1\right]$ is the discount factor.
\end{definition}
Note that the optimal solution and value function can be solved under the iterative learning framework\cite{10271561}. However, the online convergence property is subject to learning efficiency.
Motivated by the concept of receding-horizon from MPC, as discussed in~\cite{xu2018learning}, we divide $V_{\infty}(\tilde{x}_k)$ into multiple sub-problems within the prediction horizon $[k,k+N]$, where $N$ is the prediction horizon length. The value function in the prediction horizon is expressed by
\begin{equation}\label{performance_index}
\left\{\begin{aligned}
	&V(\hat{\tilde{x}}_{\tau})=\mathbb{E}\left[F(\hat{\tilde{x}}_{k+N})+\sum_{j=\tau}^{k+N-1} L(\hat{\tilde{x}}_j, \tilde{u}(\hat{\tilde{x}}_{j}))\right], \\
	&V(\hat{\tilde{x}}_{k+N})=\mathbb{E}\left[F(\hat{\tilde{x}}_{k+N})\right],
\end{aligned}\right.
\end{equation}
where $\mathbb{E}[\cdot]$ denotes the expectation value, and $F(\hat{\tilde{x}}_{k+N})$ is the terminal cost function, defined by $F(\hat{\tilde{x}}_{k+N}) = \hat{\tilde{x}}_{k+N}^{\top} P \hat{\tilde{x}}_{k+N}$, where $P\in\mathbb{R}^{n_{\mathcal{K}}\times n_{\mathcal{K}}}$ is the terminal penalty matrix, which can be determined in the same way as \cite{xu2018learning}.

The value function $V(\hat{\tilde{x}}_{\tau})$ in Eq.~\eqref{performance_index} is written by
\begin{equation}\label{performance_index_no_iter}
	V(\hat{\tilde{x}}_{\tau})= \mathbb{E}\left[L(\hat{\tilde{x}}_\tau, \tilde{u}(\hat{\tilde{x}}_\tau))+\gamma V(\hat{\tilde{x}}_{\tau+1})\right].
\end{equation}

According to the Bellman’s optimality principle, the optimal policy $\tilde{u}^*$ minimizes $V(\hat{\tilde{x}}_{\tau})$ in the prediction horizons; i.e.,
\begin{equation}\label{optimal_value_func}\resizebox{0.43\textwidth}{!}{$
	\left\{\begin{aligned}
		&V^*(\hat{\tilde{x}}_{\tau})=\mathbb{E}\left[L(\hat{\tilde{x}}_\tau, \tilde{u}^*(\hat{\tilde{x}}_\tau))+\gamma V^*(\hat{\tilde{x}}_{\tau+1})\right], \tau\in[k+N-1]\\
		&V^*(\hat{\tilde{x}}_{k+N})=\mathbb{E}\left[F(\hat{\tilde{x}}_{k+N})\right].
	\end{aligned}\right.$}
\end{equation}
\begin{definition}
For the critic network, the optimal value function in the \textbf{finite horizon} is re-defined by $\partial V^*(\hat{\tilde{x}}_\tau)/\partial \hat{\tilde{x}}_{\tau}=0$, i.e,

\begin{equation}\label{optima_lambda}
	\resizebox{0.45\textwidth}{!}{$
		\lambda^*(\hat{\tilde{x}}_{\tau})=
		\left\{\begin{array}{cc}
			2Q\hat{\tilde{x}}_{\tau}+ \gamma A_{d,\tau}^\top \lambda^*(\hat{\tilde{x}}_{\tau+1})+\dfrac{\partial h(\hat{x}_\tau)}{\partial \hat{\tilde{x}}_\tau}, & \tau \in[k, k+N-1]\\
			2P\hat {\tilde{x}}_{k+N}+{\partial h({\hat{x}}_{k+N})}/{\partial \hat{\tilde{x}}_{k+N}}, & \tau=k+N
		\end{array}\right.
		$}
\end{equation}

With the optimal value function $\lambda^*(\hat{\tilde{x}}_\tau)$, one can compute the optimal solution by setting ${\partial V^*(\hat{\tilde{x}}_\tau)}/{\partial \tilde{u}(\hat{\tilde{x}}_\tau)}=0$. Then, for the actor network, the optimal solution in the \textbf{finite horizon} is defined by 
\begin{equation}\label{optimal_action}
	\tilde{u}^*(\hat{\tilde{x}}_\tau) = -\dfrac{1}{2}\gamma R^{-1}B_{d,\tau}^\top\lambda^*(\hat{\tilde{x}}_{\tau+1}), \tau\in[k,k+N-1].
\end{equation}
\end{definition}
\begin{remark}
	The optimal solution $\tilde{u}^*(\hat{\tilde{x}}_\tau)$ and the optimal value function $\lambda^*(\hat{\tilde{x}}_{\tau})$ are difficult to be solved analytically. This is because $\lambda^*(\hat{\tilde{x}}_{\tau+1})$ is not available; see Eqs.~\eqref{optima_lambda} and~\eqref{optimal_action}. In the following, the optimal policies of the actor and critic networks are learned iteratively.
\end{remark}
\begin{algorithm}[t]
\small
\caption{VF-LPC algorithm}
\label{algorithm_moving_obstacle}
\LinesNumbered
\SetKwFunction{FSum}{LPC}
\SetKwFunction{FSumm}{QuantizedSparseGP}
\SetKwFunction{FSummm}{Main()}
\SetKwComment{Comment}{//}{}
\KwIn{Obstacle information, robot's states, the desired path. }  

\KwOut{The final safe trajectory.}

Initialize ${W}_{c,1:N}^1,{W}_{a,1:N}^1$;

\SetKwProg{Fn}{Function}{:}{}
\Fn{\FSum}{
	Initialize the coefficient $\mu_0$ with $0$.
	
	\If{$(X,Y)\in\mathcal{E}_p$}{$\mu_o\leftarrow$ a positive real number.}
	
	\For{$i=1,\cdots, i_{\max}$ (maximuim iterations)}{
		
		\tcp{\textit{Prediction Horizon}}
		
		\For{ $\tau= k, \cdots, k+N$}{
			
			Compute $\hat{\tilde{u}}^i(\hat{\tilde{x}}_\tau)$ and $\hat{\lambda}^i(\hat{\tilde{x}}_\tau)$ by Eq.~\eqref{hat_u}.
			
			Update the system dynamics~\eqref{identification_form_dynmaic} and Jacobian matrices~\eqref{JocabianMatrices} with sparse GP.
			
			Compute $\epsilon_{c}^{i}(\hat{\tilde{x}}_\tau)$ and $\epsilon_{a}^{i}(\hat{\tilde{x}}_\tau)$ with~\eqref{epsilon_c} and~\eqref{epsilon_a}.
			
			Update the weights ${W}_{c,\tau}^{i}$ and ${W}_{a,\tau}^{i}$ with~\eqref{WaWc_update}.
			
			Apply $\hat{\tilde{u}}^i(\hat{\tilde{x}}_\tau)$ to the system~\eqref{discrete-time-nonlinear-system-dynmaics}.
		}
		
		\If {the weights converge}
		{
			${W}_{c,k:k+N}^1\leftarrow W_{c,k:k+N}^i$, ${W}_{a,k:k+N}^1\leftarrow W_{a,k:k+N}^i$.
			
			Return the \emph{sequence} $\{\hat{\tilde{u}}^i(\tilde{x}_k),\cdots,\hat{\tilde{u}}^i(\hat{\tilde{x}}_{k+N})\}$.
		}
		
	}
	
}

\textbf{end}

\SetKwProg{Fn1}{Function}{:}{}
\Fn{\FSumm}{
	Select a representative sample set $\mathcal{D}_{\rm{SGP}}$ from $\mathcal{D}_{\rm{GP}}$.
	
	Optimize the hyper-parameters for online compensation. 
}

\textbf{end}

\SetKwProg{Fn1}{Function}{:}{}
\Fn{\FSummm}{
	\While {not reaching the destination}{
		
		Get the current state $x_0$ of the system. 
		
		\If{reaching the end of the guiding trajectory}{
			
			Generate the guidance (the desired trajectory) for the LPC module with Eqs.~\eqref{guiding_vector_dynamic_constraint} and~\eqref{xi_r}.
			
		}
		
		\If{the online dictionary $\mathcal{D}_{\rm{GP}}$ has been collected}{Perform \FSumm.}
		
		\tcp{\textit{Control Horizon}}
		Perform \FSum and apply the 1st solution of the $\emph{sequence}$ to the real-world robot.
	}
}
\textbf{end}

\end{algorithm}

 At the $i$-th iteration, the two optimal policies of the actor and critic networks are approximated by two kernel-based-network structures, i.e.,
\begin{subequations}\label{hat_u}
\begin{align}\label{hat_u_lambda}
	\hat{\tilde{u}}^i(\hat{\tilde{x}}_\tau)&=({W}_{a,\tau}^i)^{\top}\Phi(\hat{\tilde{x}}_{\tau}),\\
	\label{hat_lambda}
	\hat{\lambda}^i(\hat{\tilde{x}}_\tau)&=({W}_{c,\tau}^i)^{\top}\Phi(\hat{\tilde{x}}_{\tau}),
\end{align}
\end{subequations}
where ${W}_{a,\tau}\in\mathbb{R}^{n_{\Phi}\times n_u}$ and ${W}_{c,\tau}\in\mathbb{R}^{n_{\Phi}\times n_{\mathcal{K}}}$ are the weights of the actor and critic networks, respectively, and $n_{\Phi}$ denotes the dimension of the \emph{RL training} dictionary, which is obtained by employing the ALD strategy\cite{10271561} to select representative samples from the robot's state space. 
For the current state $\hat{\tilde{x}}_\tau$, the basis function in Eq.~\eqref{hat_u} is constructed by
\begin{equation*}\label{MultiKernelFeature}
\Phi(\hat{\tilde{x}}_\tau)=\left[k(\hat{\tilde{x}}_\tau, c_{1}), \ldots, k(\hat{\tilde{x}}_\tau, c_{n_{\Phi}})\right]^{\top}\in\mathbb{R}^{n_{\Phi}},
\end{equation*}
where $c_1,\cdots, c_{n_{\Phi}}$ are elements in the \emph{RL training} dictionary and $k(\cdot,\cdot):\mathbb{R}^{n_{\mathcal{K}}}\times\mathbb{R}^{n_{\mathcal{K}}}\rightarrow\mathbb{R}$ is the Gaussian kernel function.

With the estimated value function in~\eqref{hat_lambda} and consistent with~\eqref{optimal_action}, we can obtain the target optimal solution at the $i$-th iteration and the $\tau$-th horizon by
\begin{equation}\label{approximated_optimal_action}
\tilde{u}^i(\hat{\tilde{x}}_\tau) = -\dfrac{1}{2}\gamma R^{-1}B_{d,\tau}^\top\hat\lambda^i(\hat{\tilde{x}}_{\tau+1}), \tau\in[k,k+N-1].
\end{equation}

Correspondingly,  for $\tau\in[k,k+N]$, the value function $V^{i+1}(\hat{\tilde{x}}_{\tau})$ in Eq.~\eqref{performance_index} is written by
\begin{equation}\label{performance_index1}
		V^{i+1}(\hat{\tilde{x}}_{\tau})= \mathbb{E}\left[L(\hat{\tilde{x}}_\tau, \tilde{u}^i(\hat{\tilde{x}}_\tau))+\gamma V^{i}(\hat{\tilde{x}}_{\tau+1})\right].
\end{equation}

Note that $\hat{\tilde{x}}_{\tau+1}$ is obtained by the identified system dynamics~\eqref{identification_form_dynmaic}. Substituting Eq.~\eqref{hat_lambda} into Eq.~\eqref{optima_lambda}, the target value function becomes
\begin{equation}\label{approximated_optimal_valuefunction}
\resizebox{0.50\textwidth}{!}{$
	\lambda^{i+1}(\hat{\tilde{x}}_{\tau})=
	\left\{\begin{array}{cc}
		2Q\hat{\tilde{x}}_{\tau}+ \gamma A_{d,\tau}^\top \hat{\lambda}^{i}(\hat{\tilde{x}}_{\tau+1})+\dfrac{\partial h(\hat{x}_\tau)}{\partial \hat{\tilde{x}}_\tau}, & \tau \in[k, k+N-1]\\
		2P\hat {\tilde{x}}_{k+N}+{\partial h({\hat{x}}_{k+N})}/{\partial \hat{\tilde{x}}_{k+N}}, & \tau=k+N
	\end{array}\right.
	$}
\end{equation}

The critic network aims at minimizing the error function between the \emph{target} and the \emph{approximate} value functions, i.e.,
\begin{equation}\label{epsilon_c}
\epsilon_c^{i}(\hat{\tilde{x}}_\tau)=\frac{1}{2}\lVert \lambda^{i}(\hat{\tilde{x}}_\tau)-\hat \lambda^{i}(\hat{\tilde{x}}_\tau) \rVert^2.
\end{equation}
For the actor network, it minimizes the error function  between the \emph{target} and the \emph{approximate} solutions, which is defined by the following equation:
\begin{equation}\label{epsilon_a}
\epsilon_a^i(\hat{\tilde{x}}_\tau)=\frac{1}{2}\lVert \tilde{u}^{i}(\hat{\tilde{x}}_\tau)-\hat {\tilde{u}}^{i}(\hat{\tilde{x}}_\tau) \rVert^2.
\end{equation} Therefore, one can derive the following update rules for the  critic network and the actor network, respectively:  
\begin{subequations}\label{WaWc_update}
\begin{align}
	{W}_{c,\tau}^{i+1} &= {W}_{c,\tau}^{i}-\eta_c{\partial \epsilon_c^{i}(\hat{\tilde{x}}_\tau)}/{\partial W_{c,\tau}^{i}} , \\
	{W}_{a,\tau}^{i+1} &= {W}_{a,\tau}^{i}-\eta_a {\partial \epsilon_a^{i}(\hat{\tilde{x}}_\tau)}/{\partial W_{a,\tau}^{i}},
\end{align}
\end{subequations}
where $\eta_c, \eta_a>0$ are the step coefficients of the gradients. 

The overall algorithm of VF-LPC is shown in Algorithm.~\ref{algorithm_moving_obstacle}. Also, the convergence theorem of the Eb-RHRL algorithm is given below and its proof is presented in Section~\ref{Section_TheoreticAnalysis}. In the following, expressions similar to $\{\lambda^i\}$ and $\{\tilde{u}^i\}$ denote the sequences in the prediction horizon of the $k$-th time instant; i.e., $$\{\lambda^i\}:=(\lambda^i(\hat{\tilde{x}}_k),\cdots,\lambda^i(\hat{\tilde{x}}_{k+N-1})),$$ and $$\{\tilde{u}^i\}:=(\tilde{u}^i(\hat{\tilde{x}}_k),\cdots,\tilde{u}^i(\hat{\tilde{x}}_{k+N-1})).$$
\begin{thm}\label{thm1}
(Convergence of the Eb-RHRL). As the iteration number $i$ increases, sequences $\{\lambda^i\}$ and $\{\tilde{u}^i\}$ converge to~\eqref{optima_lambda} and~\eqref{optimal_action}, respectively; i.e., $\lim_{i\rightarrow\infty}\{\lambda^i\}=\{\lambda^*\}$ and $\lim_{i\rightarrow\infty}\{\tilde{u}^i\}=\{\tilde{u}^*\}$.
\end{thm}

{Theorem~\ref{thm1} indicates that the Eb-RHRL approach in Algorithm~\ref{algorithm_moving_obstacle} solves the optimal IMPC problems under state constraints by obtaining the near-optimal solutions.}


\section{Theoretic Analysis}\label{Section_TheoreticAnalysis}
This section presents the theoretical analyses regarding Theorem \ref{thm1}.  Building upon the proofs in \cite{heydari2012finite,al2008discrete,xu2018learning}, we prove that the Eb-RHRL algorithm, integrating the game-based barrier function, converges to the optimal value function and solution, thereby solving the {{OMP}} problem.
\begin{lemma}\label{lemma0}
Given a control sequence $\{\mu^i\}$ involving random actions, $\Lambda^{i+1}(\hat{\tilde{x}}_\tau)$ is defined by
\begin{equation*}
	\Lambda^{i+1}(\hat{\tilde{x}}_\tau)=\hat{\tilde{x}}_\tau^{\top} Q \hat{\tilde{x}}_\tau+\mu^i(\hat{\tilde{x}}_\tau)^{\top} R \mu^i(\hat{\tilde{x}}_\tau)+h(\hat{x}_\tau)+\gamma\Lambda^{i}(\hat{\tilde{x}}_{\tau+1}),
\end{equation*}
where $\mu^i(\cdot):\mathbb{R}^{n_{\mathcal{K}}}\rightarrow\mathbb{R}^{n_u}$. 
Rewrite \eqref{performance_index1} as follows:
\begin{equation*}
	V^{i+1}(\hat{\tilde{x}}_{\tau})=\hat{\tilde{x}}_\tau^{\top} Q \hat{\tilde{x}}_\tau+\tilde{u}^i(\hat{\tilde{x}}_\tau)^{\top} R \tilde{u}^i(\hat{\tilde{x}}_\tau)+h(\hat{x}_\tau)+\gamma V^{i}(\hat{\tilde{x}}_{\tau+1}),
\end{equation*}
where $\tilde{u}^i(\hat{\tilde{x}}_\tau)$ is defined by~\eqref{approximated_optimal_action}. We can conclude that, if $V^0(\cdot)=\Lambda^0(\cdot)=0$, then $V^i(\hat{\tilde{x}}_\tau)\le\Lambda^i(\hat{\tilde{x}}_\tau)$ for all $i$.
\end{lemma}
\begin{proof}\let\qed\relax
The sequence $\{\tilde{u}^i\}$ is used to minimize $V^i(\hat{\tilde{x}}_{\tau})$, while all elements in $\{\mu^i\}$ are random. Under the condition of $V^0(\cdot)=\Lambda^0(\cdot)=0$, $V^i(\hat{\tilde{x}}_\tau)\le\Lambda^i(\hat{\tilde{x}}_\tau)$ holds, for all  $i$. $\hfill\blacksquare$
\end{proof}
{Lemma~\ref{lemma0}  is now used to obtain and prove Lemma~\ref{lemma1-1}.}

\begin{lemma}\label{lemma1-1}
Let the sequence $\{V^i\}$ be defined by~\eqref{performance_index1}. By applying an initial admissible control policy, the following conclusions hold:
1)  There exists an upper bound $\bar{Z}^i(\hat{\tilde{x}}_\tau)$ such that $0 \le V^i(\hat{\tilde{x}}_\tau) \le \bar{Z}^i(\hat{\tilde{x}}_\tau)$ for all $i$.
2)  If~\eqref{optimal_action} is solvable, then there exists an upper bound $\bar{Z}^i(\hat{\tilde{x}}_\tau)$ such that
$V^i(\hat{\tilde{x}}_\tau)\le \bar{Z}^i(\hat{\tilde{x}}_\tau)$, where $V^i(\hat{\tilde{x}}_\tau)$	solves~\eqref{optimal_value_func} and $0 \le V^i(\hat{\tilde{x}}_\tau) \le V^*(\hat{\tilde{x}}_\tau) \le \bar{Z}^i(\hat{\tilde{x}}_\tau)$ for all $i$. 
\end{lemma}
\begin{proof}\let\qed\relax
To begin with, $\eta^i(\cdot):\mathbb{R}^{n_{\mathcal{K}}}\rightarrow\mathbb{R}^{n_u}$ is assumed to be an admissible policy. Let $V^0(\hat{\tilde{x}}_{\tau})=Z^0(\hat{\tilde{x}}_{\tau})=0$, where
\begin{equation*}
	Z^{i+1}(\hat{\tilde{x}}_\tau)  =\mathbb{E}\left[L(\hat{\tilde{x}}_\tau, \eta^i(\hat{\tilde{x}}_\tau))+\gamma Z^i(\hat{\tilde{x}}_{\tau+1})\right],\forall i\ge 1.
\end{equation*}
We define an upper bound
\begin{equation}\label{Zi1xk}
	\begin{aligned}
		\bar{Z}^i(\hat{\tilde{x}}_\tau)&=\mathbb{E}\left[\sum_{j=\tau}^{\tau+N-1} \gamma^{j-\tau}L(\hat{\tilde{x}}_j, \eta^i(\hat{\tilde{x}}_j))\right. \\& \ \ \ \ \  \ \ \ \ \ \ \ \ \ \left.+\gamma^{N}Z^{i-(N-1)}(\hat{\tilde{x}}_{\tau+N})\right] \\
		& =\mathbb{E}\left[\sum_{j=\tau}^{\tau+N-1} \gamma^{j-\tau}L(\hat{\tilde{x}}_j, \eta^i(\hat{\tilde{x}}_j))+\gamma^{N}F(\hat{\tilde{x}}_{\tau+N})\right]
	\end{aligned}
\end{equation}
with the dynamic model~\eqref{identification_form_dynmaic} and
$L(\hat{\tilde{x}}_j,\eta^i(\hat{\tilde{x}}_j))=\hat{\tilde{x}}_j^{\top}Q\hat{\tilde{x}}_j+(\eta^i(\hat{\tilde{x}}_j))^{\top} R \eta^i(\hat{\tilde{x}}_j)+h(\hat{x}_{j})$. When propagating $Z^{i+1}(\hat{\tilde{x}}_\tau)$ from current stage $\tau$ to $\tau+N$, two cases exist and analyzed below.

\textit{Case} 1 ($i\le N-1$):
\begin{equation}\label{Zi1xkzi}
	\begin{aligned}
		Z^{i+1}(\hat{\tilde{x}}_\tau) & =\mathbb{E}\left[L(\hat{\tilde{x}}_\tau, \eta^i(\hat{\tilde{x}}_\tau))+\gamma Z^i(\hat{\tilde{x}}_{\tau+1})\right] \\
		& =\mathbb{E}\left[\sum_{j=\tau}^{\tau+1} \gamma^{j-\tau}L\left(\hat{\tilde{x}}_j, \eta^i(\hat{\tilde{x}}_j)\right)+\gamma^2 Z^{i-1}(\hat{\tilde{x}}_{\tau+2})\right] \\
		&\ldots\\
		& =\mathbb{E}\left[\sum_{j=\tau}^{\tau+i} \gamma^{j-\tau}L(\hat{\tilde{x}}_j, \eta^i(\hat{\tilde{x}}_j))\right].
	\end{aligned}
\end{equation}
Since $h(\hat{x}_j)\ge 0$, by applying an admissible control $\eta^i(\cdot)$ to the system, we can obtain
\begin{equation}\label{StepCostBounded}
	\sum_{j=\tau}^{\tau+i_1} \gamma^{j-\tau}L(\hat{\tilde{x}}_j, \eta^i(\hat{\tilde{x}}_j))<\sum_{j=\tau}^{\tau+i_2} \gamma^{j-\tau}L(\hat{\tilde{x}}_j, \eta^i(\hat{\tilde{x}}_j)),
\end{equation}
where $i_1<i_2\in[0,N]$.
With Eqs.~\eqref{Zi1xkzi} and~\eqref{StepCostBounded}, we have
\begin{equation}\label{case1}\resizebox{0.85\hsize}{!}{$
	\begin{aligned}
		&Z^{i+1}(\hat{\tilde{x}}_\tau)\\&=\mathbb{E}\left[\sum_{j=\tau}^{\tau+i} \gamma^{j-\tau}L(\hat{\tilde{x}}_j, \eta^i(\hat{\tilde{x}}_j))\right]\\&< \mathbb{E}\left[\sum_{j=\tau}^{\tau+N-1}\gamma^{j-\tau}L(\hat{\tilde{x}}_j, \eta^i(\hat{\tilde{x}}_j))+\gamma^{N}F(\hat{\tilde{x}}_{\tau+N})\right]\\&=\bar{Z}^i(\hat{\tilde{x}}_\tau).
	\end{aligned}$}
\end{equation}

\textit{Case} 2 ($i\ge N$):
\begin{equation}\label{case2}\resizebox{0.85\hsize}{!}{$
		\begin{aligned}
			& Z^{i+1}(\hat{\tilde{x}}_\tau) \\
			& =\mathbb{E}\left[\sum_{j=\tau}^{\tau+N-1} \gamma^{j-\tau}L(\hat{\tilde{x}}_j, \eta^i(\hat{\tilde{x}}_j))+\gamma^{N}F(\hat{\tilde{x}}_{\tau+N})\right]=\bar{Z}^i(\hat{\tilde{x}}_\tau) .
		\end{aligned}$}
\end{equation}

Combining Eqs.~\eqref{case1} and~\eqref{case2}, $Z^{i+1}(\hat{\tilde{x}}_\tau)\le\bar{Z}^i(\hat{\tilde{x}}_\tau)$ holds.

According to Lemma~\ref{lemma0}, setting the sequence $\{\mu^i\}$ as $\{\eta^i\}$ and the sequence of the value function $\{\Lambda^i\}$ as $\{Z^i\}$, it follows that $V^{i}(\hat{\tilde{x}}_\tau)\le Z^{i}(\hat{\tilde{x}}_\tau)\le \bar{Z}^i(\hat{\tilde{x}}_\tau)$ for all $i$, which proves 1). By setting $\{\eta^{i}\}$ as $\{\tilde{u}^*\}$, one has $0\le V^{i}(\hat{\tilde{x}}_\tau)\le V^*(\hat{\tilde{x}}_\tau)\le \bar{Z}^i(\hat{\tilde{x}}_\tau)$, and part 2) of Lemma~\ref{lemma1-1} is now proved.
$\hfill\blacksquare$
\end{proof}

\begin{lemma}\label{Monotonicity_lemma}
	{\cite[Monotonicity property]{xu2018learning}} Let $V^i(\hat{\tilde{x}}_\tau)$ be defined by~\eqref{performance_index1}. With $V^0(\hat{\tilde{x}}_\tau)=0$, $V^{i+1}(\hat{\tilde{x}}_\tau)\ge V^i(\hat{\tilde{x}}_\tau)$ holds.
\end{lemma}
\begin{proof}\let\qed\relax
	The proof is given in\cite{heydari2012finite}.
\end{proof}
Under Lemma~\ref{lemma1-1}, we now arrive at the proof of Theorem~\ref{thm1}.

\begin{proof}[Proof of Theorem \ref{thm1}]\let\qed\relax
Following the part 2) of Lemma~\ref{lemma1-1}, one has $\lim_{i\rightarrow\infty}V^{i}(\hat{\tilde{x}}_\tau)\le V^*(\hat{\tilde{x}}_\tau)$.	According to Lemmas~\ref{lemma1-1} and~\ref{Monotonicity_lemma},   $V^{\infty}(\hat{\tilde{x}}_\tau)=\bar{Z}^i(\hat{\tilde{x}}_\tau)\ge V^*(\hat{\tilde{x}}_\tau)$ holds for some $V^{\infty}(\hat{\tilde{x}}_\tau)$. The two aspects indicate that  $\lim_{i\rightarrow\infty}V^{i}(\hat{\tilde{x}}_\tau)= V^*(\hat{\tilde{x}}_\tau)$ and $\lim_{i\rightarrow\infty}\tilde{u}^i(\hat{\tilde{x}}_\tau)= \tilde{u}^*(\hat{\tilde{x}}_\tau)$. We further derive $\lim_{i\rightarrow\infty}\{V^i\}=\{V^*\}$ and $\lim_{i\rightarrow\infty}\{\tilde{u}^i\} = \{\tilde{u}^*\}$.
Accordingly, $\lambda^i(\hat{\tilde{x}}_\tau)$ designed by Eq.~\eqref{approximated_optimal_valuefunction} converges to $\lambda^{*}(\hat{\tilde{x}}_\tau)$ and the sequence $\{\lambda^i\}$ converges to  $\{\lambda^*\}$ eventually, as $i\rightarrow\infty$. $\hfill\blacksquare$
\end{proof}
\begin{thm}\label{thm_stablility}
	(Stability of the VF-LPC). The closed-loop system~\eqref{Dynamics_for_planning} is stable.
\end{thm}
\begin{proof}
	Define the Lyapunov candidate function $W(\tau):=V(\hat{\tilde{x}}_\tau)$ for the closed-loop system~\eqref{Dynamics_for_planning}. One has 
	\begin{equation*}
		W(\tau)=\mathbb{E}\left[L(\hat{\tilde{x}}_\tau, \tilde{u}(\hat{\tilde{x}}_\tau))+ V(\hat{\tilde{x}}_{\tau+1})\right].
	\end{equation*}
	It can be deduced that the Lyapunov function equals zero at the origin of the system state and remains strictly positive at all non-origin points.
	In the following, we will prove that $\Delta W=W(\tau+1)-W(\tau)<0, \forall \tau\in[k,k+N-1]$. For $k\le\tau< k+N-1$, $W(\tau+1)-W(\tau)\le-\mathbb{E}[L(\hat{\tilde{x}}_\tau, \tilde{u}(\hat{\tilde{x}}_\tau))]$. For $\tau=k+N-1$, one can set a function \( F(\cdot) \) such that the inequality $ \mathbb{E}\left[W(\tau) \right]=\mathbb{E}\left[ L( \hat{\tilde{x}}_{\tau}, \tilde{u}(\hat{\tilde{x}}_{\tau}))+V(\hat{\tilde{x}}_{\tau+1})  \right]\leq \mathbb{E}\left[ F(\hat{\tilde{x}}_{\tau}) \right]$ holds, thereby deriving $W(\tau+1)=\mathbb{E}(L(\hat{\tilde{x}}_{\tau+1}, \tilde{u}(\hat{\tilde{x}}_{\tau+1})))\le F(\hat{\tilde{x}}_{\tau+1})=W(\tau)-L(\hat{\tilde{x}}_\tau, \tilde{u}(\hat{\tilde{x}}_\tau))$. Since $L(\hat{\tilde{x}}_\tau, \tilde{u}(\hat{\tilde{x}}_\tau))\ge0$ and $L(\hat{\tilde{x}}_\tau, \tilde{u}(\hat{\tilde{x}}_\tau))\ge 0$, it is shown that Lyapunov function exhibits a decreasing behavior in the prediction horizon.
	
	In the following, we will demonstrate that the value function of the system~\eqref{Dynamics_for_planning} at two distinct time instants decreases monotonically.
	Let $\bar{L}(\hat{\tilde{x}}_{k+N})=L(\hat{\tilde{x}}_{k+N},K\hat{\tilde{x}}_{k+N})$, where $K$ is defined as the feedback gain matrix. For the brevity of analysis, define $V(\tilde{x}_k, N)=\sum_{\tau=k}^{k+N-1} L(\hat{\tilde{x}}_\tau, u(\tau))+V_f(\hat{\tilde{x}}_{k+N})$. Now one could set proper terminal value functions to make the condition $V(\hat{\tilde{x}}_{k+N+1})-V(\hat{\tilde{x}}_{k+N})\le -L(\hat{\tilde{x}}_{k+N})$ hold. Then we have 
	\begin{equation}\label{optimal_valuefuc_predictionhorizon}
		 \begin{aligned}
			&V^*(\hat{\tilde{x}}_{k+1},N) \\& =\sum_{\tau=k+1}^{k+N} L\left(\hat{\tilde{x}}_\tau, u^*(\hat{\tilde{x}}_{\tau})\right)+V_f(\hat{\tilde{x}}_{k+N+1}) \\
			& \leq V(\hat{\tilde{x}}_{k+1}, N) \\
			& =\sum_{\tau=k+1}^{k+N-1} L\left(\hat{\tilde{x}}_{\tau}, \tilde{u}^*(\hat{\tilde{x}}_{\tau})\right)+\bar{L}(\hat{\tilde{x}}_{k+N})+V_f(\hat{\tilde{x}}_{k+N+1}) \\
			& =V^*(\tilde{x}_{k}, N)-L\left(\tilde{x}_{k}, \tilde{u}^*(\tilde{x}_{k})\right)+\bar{L}(\hat{\tilde{x}}_{k+N}) \\
			&\,\,\,\,\,\,\, +\left(V_f(\hat{\tilde{x}}_{k+N+1})-V_f(\hat{\tilde{x}}_{k+N})\right).
		\end{aligned}
	\end{equation}

	As $\bar{L}(\hat{\tilde{x}}_{k+N})\ge 0$, one has

	\begin{equation}\label{condition_of_stability}
		\begin{aligned}
			&V_f(\hat{\tilde{x}}_{k+N+1})-V_f(\hat{\tilde{x}}_{k+N}) \\
			& \leq-\bar{L}(\hat{\tilde{x}}_{k+N}) \\
			& =-L^{\top}(\hat{\tilde{x}}_{k+N})\left(Q+K^{\top} R K\right) L(\hat{\tilde{x}}_{k+N}).
		\end{aligned}
	\end{equation}

	Substituting Eq.~\eqref{condition_of_stability} into Eq.~\eqref{optimal_valuefuc_predictionhorizon} yields
	\begin{equation}\label{stability_results}
		\begin{aligned}
			V^*(\hat{\tilde{x}}_{k+1},N)\le V^*(\tilde{x}_{k},N)-L(\tilde{x}_{k}, \tilde{u}^*(\tilde{x}_{k})).
		\end{aligned}
	\end{equation}
	Therefore, we can conclude that $V^*(\tilde{x}_{k},N)$ is monotonically decreasing over time $k$. One has
	\begin{equation}
		V^*(\hat{\tilde{x}}_{\infty}, N) \leq V^*(\tilde{x}_{k}, N)-\sum_{\tau=k}^{\infty} L\left(\hat{\tilde{x}}_{\tau}, \tilde{u}^*(\hat{\tilde{x}}_{\tau})\right)
	\end{equation}
	and $\forall k\ge 0, V^*(\tilde{x}_{k}, N)\ge 0$. It can be obtained that when $k\rightarrow\infty$, we have $L\left(\tilde{x}_{k}, \tilde{u}^*(\tilde{x}_{k})\right)\rightarrow 0$. Moreover, since $Q, R$ are positive definite matrices, it follows that $\tilde{x}_k\rightarrow \mathbf{0}, u^*(\tilde{x}_k)\rightarrow \mathbf{0}$, as $k\rightarrow \infty$. Therefore, system~\eqref{Dynamics_for_planning} is asymptotically stable. When the modeling error is sufficiently small, the original system~\eqref{nonlienar_nonaffline_system} is also asymptotically stable.

\end{proof}
\section{Simulation and Experimental Results}\label{Comparision}
In this section, we present the simulation and experimental results to validate our proposed VF-LPC approach.
The algorithm was deployed on a computer running Windows 11 with an Intel Core i7-11800H @2.30GHZ CPU. We first compared VF-LPC with advanced IMPC methods in CarSim software. Then the effectiveness of the online model compensating strategy is verified through a tracking control task. The simulation results on a quadrotor UAV were performed. We also tested the VF-LPC approach on an actual ground vehicle of Hongqi EHS-3. The platform is shown in the right part of Fig.~\ref{fig_DRHACL_VehicleControl}. 
\subsection{Path Planning in An Obstacle-Dense Environment}
To evaluate the effectiveness of the VF-LPC approach, we compare it with advanced planning approaches in a scenario containing static obstacles; see Fig.~\ref{fig_ScenarioI}. We compare VF-LPC with the commonly used planning approaches, such as Hybrid A*\cite{dolgov2010path}, RRT*-CFS\cite{leu2021efficient}, Timed Elastic Bands (TEB)\cite{rosmann2017kinodynamic} and OBTPAP\cite{li2021optimization} which iteratively solves the time-optimal motion planning problem. In these tests, we use the analytic kinematic vehicle model with the state variable $[X,Y,\psi]$, representing the $X$-, $Y$-coordinates and the heading, respectively. All the approaches are performed in MATLAB 2023a. 

\subsubsection{Evaluation metrics}
To evaluate the planning methods, we employ various performance metrics. The first metric is the trajectory length from the initial point to the destination, denoted as $L_p$; i.e., ${L}_p = \sum_{k=1}^{M-1} \lVert(X_{k+1},Y_{k+1})-(X_{k},Y_{k})\rVert$, where $M$ is the number of trajectory points.
The second metric $t_{\mathrm{cal}}$ is the $\mathrm{calculation\ time}$ of trajectory planning. The third metric pertains to $\mathrm{SD}$, serving to quantify whether the planned trajectories have a safe distance from obstacles.

\subsubsection{Analyses of the results}
As shown from TABLE~\ref{tab_fig_ScenarioII}, the paths planned by methods VF-LPC and TEB exhibit similar lengths, albeit shorter than OBTPAP. Though the trajectory length of Hybrid A* is the shortest, it is deemed unsafe due to the lack of maintaining a safe distance while avoiding the first obstacle. Regarding computational efficiency, VF-LPC demonstrates the best performance among the methods.

\begin{table}[!htb]\centering
\setlength\tabcolsep{3.5pt}
\renewcommand\arraystretch{1.0}
\caption{The Comparative Analysis of Planning Results using Different Methods in the Obstacle-dense Environment.}\label{tab_fig_ScenarioII}
\begin{threeparttable}\scriptsize 
	\begin{tabular}{ccccccccc}
		\toprule 
		{Metrics}&{VF-LPC}&Hybrid A*\cite{dolgov2010path}&RRT*-CFS\cite{leu2021efficient}&{TEB\cite{rosmann2017kinodynamic}}&{OBTPAP\cite{li2021optimization}}\\ \midrule 
		${L}_p$ (\rm{m})&\,\,$47.44$&\,\,$\textbf{46.16}$\,&$54.21$&$48.34$&$47.54$\\
		$t_{\mathrm{cal}}$ (\rm{s})&\,\,\,\,$\mathbf{0.29}$&\,\,\,$2.29$&$50.28$&\,\,\,$4.80$&\,\,\,$5.72$\\
		$\mathrm{SD}$?&$\checkmark$&$\times$&$\checkmark$&$\checkmark$&$\checkmark$\\
		\bottomrule 
	\end{tabular}
	The bold indicates the best result of each metric.
\end{threeparttable}
\end{table}
\begin{figure}[!htb]
\centering\includegraphics[width=3.5in]{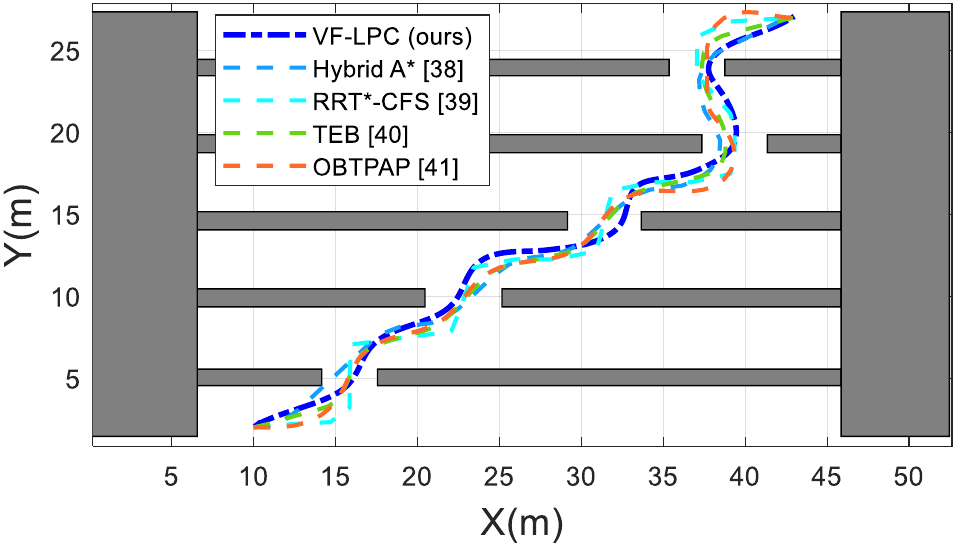}
\caption{The planned results using different comparison methods.}
\label{fig_ScenarioI}
\end{figure}
\subsection{Collision-Avoidance Simulation in High-Fidelity CarSim}
Tracking a desired path while avoiding static and moving obstacles is a fundamental task. To demonstrate the superiority of VF-LPC, we compare it with MPC-CBF\cite{zeng2021safety}, LMPCC\cite{brito2019model},  RHRL-KDP\cite{9756946}, and CFS\cite{liu2018convex} under different metrics. 
\subsubsection{Simulation settings}
All the comparison methods utilize the system dynamics model~\eqref{vehicle_dynamics} and solve the nonlinear optimization problem using the IPOPT solver\cite{wachter2006implementation} within the CasADi framework\cite{andersson2019casadi}, which describes the vehicle dynamics as follows:
\begin{equation}\label{vehicle_dynamics}\resizebox{0.85\hsize}{!}{$
	\left[ \begin{array}{l}
		\dot{X}\\
		\dot{Y}\\
		\dot{\psi}\\
		\dot{v}_x\\
		\dot{v}_y\\
		\dot{\omega}\\
	\end{array} \right] =\left[ \begin{array}{c}
		v_x\cos \psi -v_y\sin \psi\\
		v_x\sin \psi +v_y\cos \psi\\
		\omega\\
		v_y\omega +a_x\\
		2C_{af}( \frac{\delta_f}{m} -\frac{v_y+l_f\omega}{mv_x}) + 2C_{ar}\frac{l_r\omega -v_y}{mv_x}  -v_x\omega\\
		\frac{2}{I_z}\left[ l_fC_{af}( \delta _f-\frac{v_y+l_f\omega}{v_x}) -l_rC_{ar}\frac{l_r\omega -v_y}{v_x} \right]\\
	\end{array} \right], $}
\end{equation}
where $x=[X, Y, \psi, v_x, v_y, \omega]^\top\in \mathbb{R}^6$ is the state vector, $X, Y$ are the global horizontal and vertical coordinates of the vehicle, respectively, $\psi$ is the yaw angle, $v_x, v_y$ denote the longitudinal and lateral velocities, respectively, $\omega$ denotes the yaw rate, $l_f$ and $l_r$ are the distances from the center of gravity (CoG) to the front and rear wheels, respectively; $C_{af}, C_{ar}$ represent the cornering stiffnesses of the front and rear wheels, respectively; $I_z$ denotes the yaw moment of inertia; $m$ is the vehicle's mass. Their values can be found in TABLE~\ref{table:VehicleParameters}. In this model, the acceleration $a_x$ and the steering angle $\delta_{f}$ are two variables of the control vector $u$, i.e., $u=[a_x, \delta_{f}]^{\top}$. The reference state is ${x_r}=[X_r, Y_r, \psi_r, v_x^r, v_y^r, \omega_r]^{\top}$.
\begin{table}[!htb]
\centering
\setlength\tabcolsep{2.5pt}
\caption{Vehicle Dynamic Parameters}\label{table:VehicleParameters}
\begin{threeparttable}
	\begin{tabular}{ccccccccc}
		\toprule 
		$m$&$I_z$&$l_f$&$l_r$&$C_{af}$&$C_{ar}$\\ \midrule 
		$2257\rm{kg}$&\ $3524.9\rm{kg\cdot m^2}$&\ $1.33\rm{m}$&\ $1.81\rm{m}$&$66900\rm{N/rad}$&$62700\rm{N/rad}$\\
		\bottomrule 
	\end{tabular}
\end{threeparttable}
\end{table}

For VF-LPC, we employed the approach outlined in \cite{xiao2023ddk} to conduct system identification by collecting vehicle motion data from the CarSim solver. Finally, a linear time-invariant system model~\eqref{equ:latent_vehicle_dynamics} is generated for the VF-LPC algorithm, where $n_{\mathcal{K}}=10$ and $n_u=2$.  To construct an error model for facilitating subsequent algorithm design, we subtract the desired state from the current state, i.e, $\tilde{x}=x-x_r$ and the desired control from system control, i.e., $\tilde{u} = [a_x,\delta_{f}]^{\top}$. 
\subsubsection{Evaluation metrics}
The desired speed $v_d$ of all the methods is set to $25$ \rm{km/h}. To evaluate the VF-LPC approach, we compare it with other IMPC algorithms under several metrics, which are Aver. S.T. (average solution time of each time step), $J_{\rm{Lat}}$ (cost of the lateral error), $J_{\rm{Heading}}$ (cost of the heading error), and $J_{\rm{Con}}$ (control cost).
Specifically, they are defined by
$J_{\rm{Lat}}=\lVert (X-X_r)\sin\psi_r-(Y-Y_r)\cos\psi_r \rVert_{q_1}^2$,
$J_{\rm{Heading}}=\lVert \psi-\psi_r \rVert_{q_2}^2$,
and $J_{\rm{Con}}= \lVert\tilde{u}\rVert_{R}^2$, where $q_1,q_2>0$ and $R\in\mathbb{R}^{n_u\times n_u}$ is definite.
We define a weighted average cost to evaluate the overall performance, which is
\begin{equation*}\label{CostFunction}
\begin{aligned}
	J_{\rm{MC}} = \begin{array}{c}\frac{1}{M}\sum_{k=1}^M (J_{k,\rm{Lat}}+J_{k,\rm{Heading}}+J_{k,\rm{Con}}),\end{array}
\end{aligned}
\end{equation*}
where $M$ is the number of waypoints of the driving trajectory. The other quantitative metrics are route length $L_p$, completion time $t_{\mathrm{com}}$, and calculation time $t_{\rm{cal}}$.
\subsubsection{Analyses of the results}
The testing scenario requires the vehicle to track a reference path (black line in Fig.~\ref{fig_ScenarioII})  while avoiding collisions with static and moving obstacles.

\begin{figure}[!htb]
\centering\includegraphics[width=3.0in]{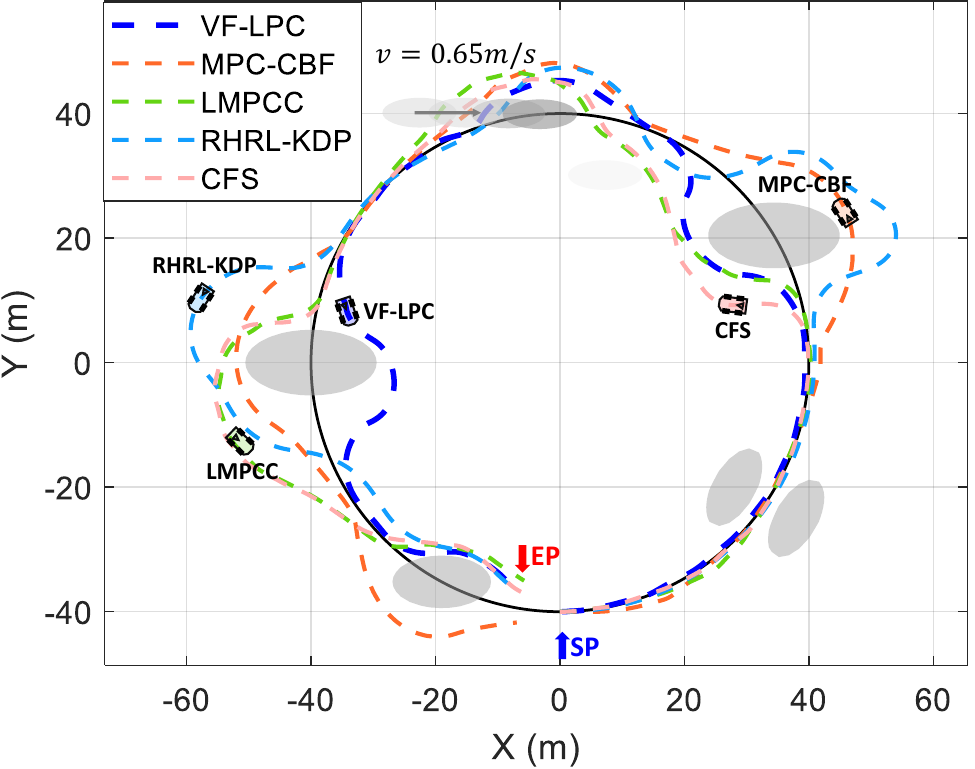}
\caption{Simulation results of avoiding static and moving obstacles for four IMPC methods. The black arrow denotes the moving direction of the dynamic obstacle, and the black thin line is the desired path. Labels `EP' and `SP' denote the endpoint and start point, respectively.}
\label{fig_ScenarioII}
\end{figure}

\begin{table}[!htb]
\centering
\setlength\tabcolsep{2.5pt}
\renewcommand\arraystretch{1.0}
\caption{Performance Evaluations of Integrated Motion Planning and Control Methods}\label{table:MetricsComparisonII}
\begin{threeparttable}\scriptsize
	\begin{tabular}{ccccccccc}
		\toprule 
		{ Metrics}&{VF-LPC}&{MPC-CBF\cite{zeng2021safety}}&{LMPCC\cite{brito2019model}}&RHRL-KDP\cite{9756946}&CFS\cite{liu2018convex}\\ \midrule 
		$J_{\rm{MC}}$&$\mathbf{61.50}$&$76.76$&$79.09$&$155.98$&$105.03$\\
		$L_p$ \rm{(m)}&$\mathbf{252.54}$\,\,\,\, &$287.45$\,\,\,&$266.76$\,\,\,&$307.29$&$273.71$\\
		$t_\mathrm{com}$ \rm{(s)}&$\mathbf{38.40}$&$41.13$&$38.88$&\,\,\,$44.72$&\,\,\,$39.03$\\
		$t_{\mathrm{cal}}$ \rm{(s)}&$\mathbf{0.006}$&\,\,\,$0.15$&\,\,\,$0.15$&\,\,\,\,\,\,\,$0.07$&\,\,\,\,\,\,$0.12$\\
		\bottomrule 
	\end{tabular}
	\footnotesize
	The bold indicates the best result of each metric.
\end{threeparttable}
\end{table}
As seen in TABLE~\ref{table:MetricsComparisonII}, VF-LPC has the lowest motion control cost $J_{\rm{MC}}$. The computational cost of VF-LPC is lower than that of MPC-CBF, LMPCC, RHRL-KDP, and CFS, and it generates the shortest trajectory and has the least completion time. As VF-LPC employs a linear time-invariant Koopman-based vehicle dynamic model to optimize the nonlinear optimal motion planning problem based on the scheme of Eb-RHRL, the computational time is the least. However, the MPC-based methods and RHRL-KDP require online solving of nonlinear optimization problems, resulting in a computational burden.

\subsection{Model Uncertainties Learning}
As mentioned in previous sections, the offline trained deep Koopman model probably does not perfectly reflect the exact system dynamics when deployed online due to external uncertainties. As a result, this could cause performance degradation both in planning and control and even safety issues. In this section, we will demonstrate the sparse GP-based learning model uncertainties for the deep Koopman model. 

The uncertainty in the dynamics of a system negatively affects the planning and control performance of robots. To eliminate the influence, a quantified sparse GP \cite{10271561} is utilized to compensate for the difference between the offline-trained deep Koopman model and the exact vehicle model. The simulated dynamical model~\eqref{vehicle_dynamics}, reveals that the vehicle states $v_x, v_y, \omega$ and $\rho_e(x_k)$ are primarily affected by dynamic parameters. Therefore, we further derive
$B_d=[\mathbf{0}_3\ I_{3}\ \mathbf{0}_{3\times n_{\rho_e}}]^\top$,
\text{and}\ $g(x,u)=g(v_x,v_y,\omega,\rho_e^{\top},a_x, \delta_{f}):\mathbb{R}^{n_{\rho_e}+5}\rightarrow\mathbb{R}^{3}$.

The algorithm was simulated in a racing track road and the desired speed $v_d$ is set to be $6\ \rm{m/s}$. Note that we collect motion data of the real vehicle (right half of Fig.~\ref{fig_DRHACL_VehicleControl}) to train for an offline nominal deep Koopman model with $n_{\rho_e}=6$ and $x=[X, Y, \psi, v_x, v_y, \omega]^\top\in \mathbb{R}^6$. The exact vehicle parameters are: $m=1257 \rm{kg}$, $I_z=1524.9 \rm{kg}\cdot\rm{m}^2$, $C_{af}=8790 \rm{N/rad}$, and $C_{ar}=30400\rm{N/rad}$. The iid process noise $w_k\in\mathbb{R}^6$ with zero mean and variance $\sigma=0.002$ is added to the vehicle dynamics. Specifically, $g(x,u)$ maps $[x,u]^\top$ to $[v_x,v_y,\omega]^\top$. 

\begin{figure}[!htp]
\centering\includegraphics[width=3.5in]{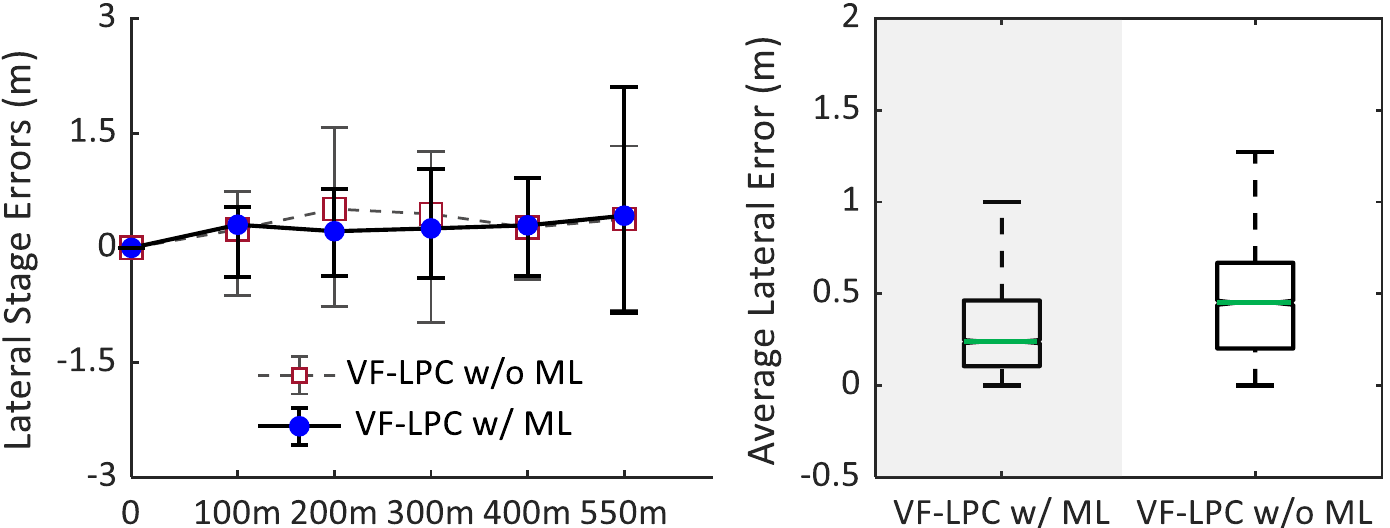}
\caption{Tracking error comparison between VF-LPC (w/o ML) and VF-LPC (w/ ML). }
\label{fig_tracking_results_w_wo_GP}
\end{figure}

To evaluate the effectiveness of our algorithm on learning model uncertainties, {VF-LPC} with and without model learning, abbreviated as {VF-LPC} (w/ ML) and {VF-LPC} (w/o ML), are separately tested on the racing road. At the initial stage, i.e., the first $100\rm{m}$, we update the dictionary online, and it is then sparsified with Approximate Linear Dependence (ALD) similarly to\cite{10271561}. With the sparsified dictionary, we train it to update the Jacobian matrices~\eqref{JocabianMatrices} and compensate for~\eqref{equ:latent_vehicle_dynamics} during the remaining miles.

From the simulation results in Fig.~\ref{fig_tracking_results_w_wo_GP}, notable improvement can be observed in the lateral stage error of VF-LPC (w/ ML) compared to VF-LPC (w/o ML). Ultimately, the average lateral error throughout the entire testing process for VF-LPC (w/ ML) is lower than that of VF-LPC (w/o ML). These results validate the capability of VF-LPC to process model uncertainty and enhance control performance.

\subsection{Validation on Unmanned Aerial Vehicles}
In this subsection, we demonstrate the effectiveness of VF-LPC on motion planning tasks of a different class of mobile robots, i.e., quadrotor UAVs, which have intricate dynamics. We use a widely-used model \cite{xu2022dynamic} to simulate the dynamics. 

Firstly, we conduct data-driven modeling of the quadrotor UAV using deep Koopman operators\cite{xiao2023ddk}, followed by employing the VF-LPC to generate the robot's maneuvers in a 3D environment, as illustrated in Fig.~\ref{fig_uav_sim_results}. The state variables of the training data consist of the XYZ coordinates and velocities ($v_x, v_y, v_z$) of the quadrotor UAV, with control variables of acceleration in the XYZ directions. A total of $148$ data, each comprising approximately $4500$ samples, were collected randomly in the state and control space for training and model validation purposes. The $6$-dimensional state vector $[X,Y,Z,v_x,v_y,v_z]^\top$ was processed through an encode layer, resulting in a final system model of $18$ state variables and $3$ control variables (i.e., $n_{\mathcal{K}}=18, n_u=3$). Then the obtained model is incorporated into VF-LPC, which focuses on planning 3-axis accelerations for attitude control. Not only the VF-LPC approach can avoid multiple static obstacles. Under suddenly appearing obstacles (see Fig.~\ref{fig_uav_sim_results}), the VF-LPC approach is also capable of handling them, and the average speed of the vehicle reaches $1 \rm{m/s}$.

The results show that the VF-LPC approach can achieve near-optimal motion planning for robots with offline-trained data-driven dynamics.
\begin{figure}[!htp]
\centering\includegraphics[width=3.5in]{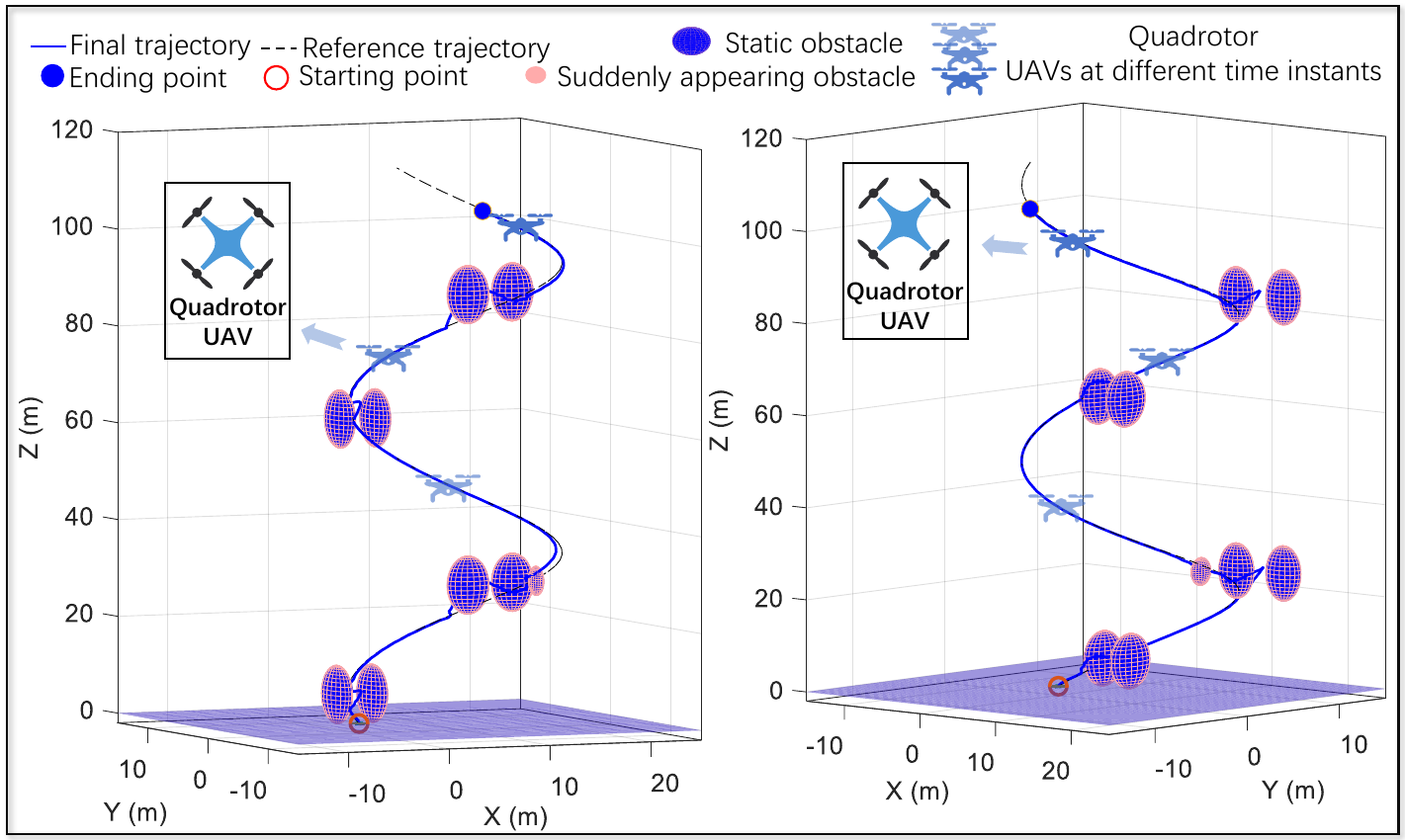}
\caption{Validating test on UAVs: Avoiding multiple static and suddenly appearing obstacles when tracking a desired trajectory. }
\label{fig_uav_sim_results}
\end{figure}
\subsection{Real-World Experiments}
To further validate the effectiveness of VF-LPC, real-world vehicular experiments were conducted on the Hongqi E-HS3 platform, which is shown in module B-(1) of Fig.~\ref{fig_DRHACL_VehicleControl}.

At each time instance within the predictive horizon, the desired trajectory and obstacles are initially transformed into the vehicle body's local coordinate system. Then, we employ polynomial curves to fit the desired path points, obtaining the desired path $\mathcal{P}$ with~\eqref{guiding_vector_dynamic_constraint}. Subsequently, the safe trajectory generated by the kinodynamic composite vector field is transformed back into the global coordinate system.

\subsubsection{Multiple Static Obstacles Avoidance}
In this scenario, we evaluate the obstacle avoidance capability of the VF-LPC approach. As illustrated in Fig.~\ref{fig_exp1}, the black dash line denotes the desired path. Due to the vehicle's long wheelbase of $4.9$ meters, the turning radius is large. However, the desired path is constrained and small in size, posing a significant challenge for the algorithm in terms of safe obstacle avoidance and tracking. We set the speed to $1.5 \rm{m/s}$, and it can be observed that the vehicle successfully reaches the destination while avoiding multiple obstacles. VF-LPC can plan a smooth trajectory for guiding finite-horizon actor-critic learning processes. This also reflects the effectiveness and advantages of the kinodynamic guiding vector field which satisfies the kinodynamic constraint.
\begin{figure}[!htp]
\centering\includegraphics[width=3.5in]{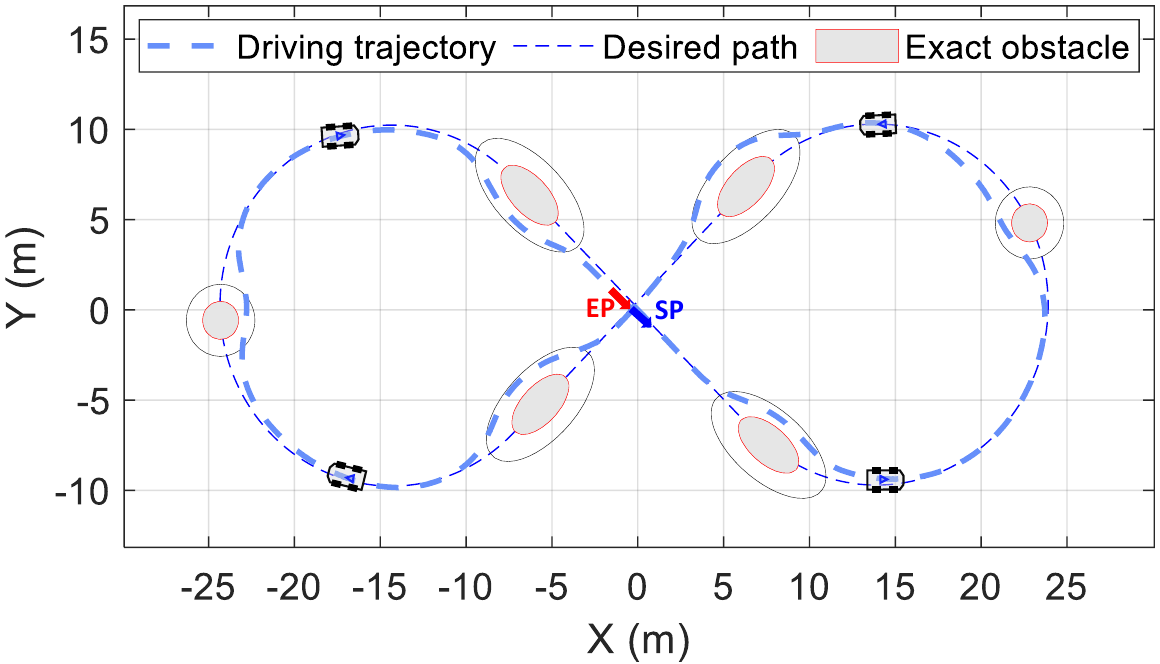}
\caption{Avoiding multiple static obstacles when tracking a size-constrained desired path. The black vehicles represent the intelligent vehicle at different time instants. Labels `EP' and `SP' denote the endpoint and starting point.}
\label{fig_exp1}
\end{figure}

\subsubsection{Moving Obstacles Avoidance}
As shown in Fig.~\ref{fig_exp2}, a human driver first drove the intelligent vehicle to generate the desired path $\mathcal{P}$ in this scenario. We have the following settings for testing our algorithm: The moving obstacles start to move at different preset velocities when the distance between them and the intelligent vehicle is less than $25 \rm{m}$, thereby validating its emergency collision avoidance capability. Using a gradient of gray, we label the obstacles' positions at different moments during their motion processes, where the darkest color represents the initial moment.

\begin{figure}[!htp]
\centering\includegraphics[width=3.5in]{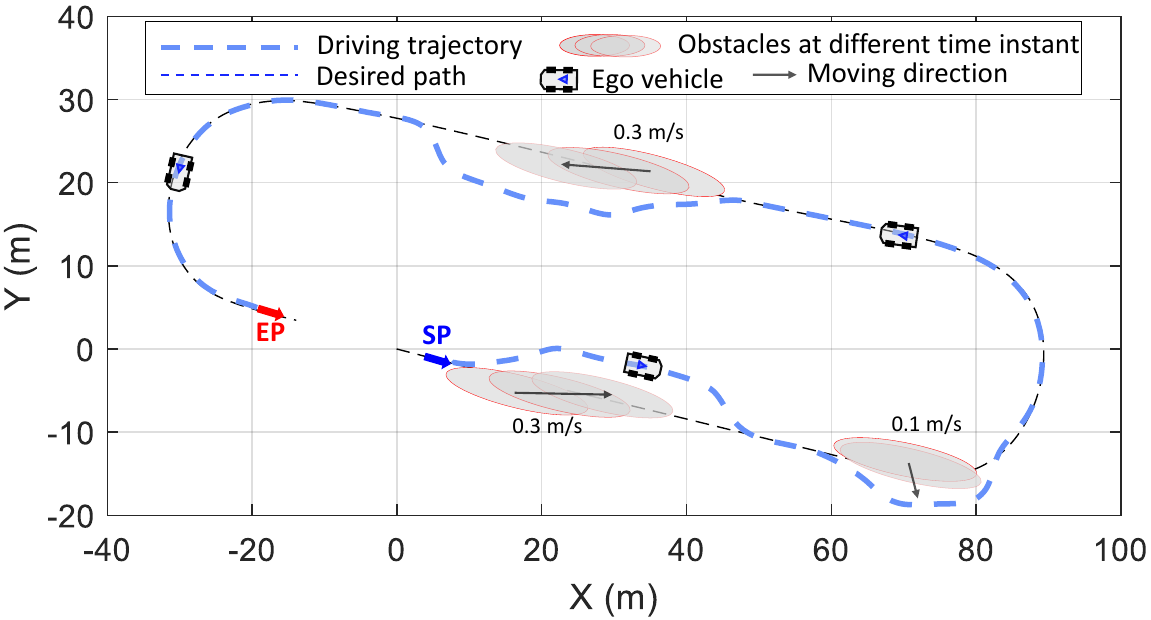}
\caption{Avoiding multiple moving obstacles when tracking a desired path. Labels `EP' and `SP' denote the endpoint and start point, respectively.}
\label{fig_exp2}
\end{figure}

From the overall tracking results in Fig.~\ref{fig_exp2}, the intelligent vehicle keeps safe distances from obstacles all the time and returns to the desired path smoothly. In addition, our method exhibits small lateral tracking errors and longitudinal velocity deviation Moreover, the maximum vehicle's speed reaches $2.4\ \rm{m/s}$. Finally, it arrives at the ending point successfully and completes the motion planning and control task.

\section{Conclusion and Future Work}\label{Conclusion}
This paper presents the Vector Field-guided Learning Predictive Control (VF-LPC) approach for mobile robots with safety guarantees, which offers a framework for the Integrated Motion Planning and Control (IMPC) of robots. VF-LPC designs kinodynamic guiding vector field for safe robot maneuvering. This is a notable improvement over the existing composite vector fields. Also, the learned deep Koopman model is updated online by sparse GP to improve safety and control performance. It is then incorporated into LPC to solve nonlinear IMPC problems.  {Rigorous theoretical analysis is provided to witness the online learning convergence.} 

VF-LPC is evaluated against motion planning methods that employ MPC and RL in high-fidelity CarSim software. The results show that VF-LPC outperforms them under metrics of completion time, route length, and average solution time. To further show the effectiveness and generalization of our proposed approach, we carried out path-tracking control tests of a mobile vehicle on a racing road to validate the model uncertainties learning capability, and we also successfully implemented the approach on quadrotor UAVs. Finally, we conducted real-world experiments on a Hongqi E-HS3 vehicle.

{Our work has several possible future directions. 1)  With an additional prediction module, VF-LPC could be promising even in higher-speed tasks. 2) In unknown environments, one can construct safer barrier functions in real-time using the information perceived by onboard sensors.}

\bibliographystyle{IEEEtran}
\bibliography{VF_LPC.bib}
\end{document}